\documentclass{article}
\usepackage{amsfonts,amssymb,amsbsy,amsmath,amsthm}
\usepackage{graphicx}
\usepackage{hyperref} 
\usepackage{natbib}
\usepackage[title,toc]{appendix}
\usepackage[nameinlink,capitalize,noabbrev]{cleveref}
\usepackage{commath}
\usepackage{algpseudocode}
\usepackage{algorithm}
\usepackage{algorithmicx, algpseudocode}
\usepackage{bbm}
\usepackage{placeins}
\usepackage{caption,subcaption}
\usepackage{booktabs}
\usepackage{multirow}
\usepackage[letterpaper, margin=1.1in]{geometry} 
\usepackage{authblk}

\newtheorem{theorem}{Theorem}
\newtheorem{lemma}[theorem]{Lemma}
\theoremstyle{definition}

\newtheorem{remark}{Remark}
\crefname{assumption}{Assumption}{Assumptions}
\crefname{lemma}{Lemma}{Lemmas}
\crefname{theorem}{Theorem}{Theorems}
\crefname{algorithm}{Algorithm}{Algorithms}
\crefname{section}{Section}{Sections}
\crefname{appendix}{Appendix}{Appendices}

\def\cD{\mathcal{D}}

\def\cI{\mathcal{I}}

\def\E{\mathbb{E}}
\def\R{\mathbb{R}}
\def\P{\mathbb{P}}
\def\bN{\mathbb{N}}
\def\uk{\underline{k}}

\def\ind{\mathbb{I}}

\def\quantile{\text{Quantile}}

\def\indsim{\stackrel{\textup{ind}}{\sim}}
\def\iidsim{\stackrel{\textup{iid}}{\sim}}
\def\tinit{t_{\mathrm{init}}}

\title{Online Conformal Inference with Retrospective Adjustment for Faster Adaptation to Distribution Shift}
\author{Jungbin Jun and Ilsang Ohn\footnote{Corresponding author. Email: ilsang.ohn@inha.ac.kr}}
\affil{Department of Statistics, Inha University}
\date{\today}

\begin{document}

\maketitle

\begin{abstract}
Conformal prediction has emerged as a powerful framework for constructing distribution-free prediction sets with guaranteed coverage assuming only the exchangeability assumption. However, this assumption is often violated in online environments where data distributions evolve over time. Several recent approaches have been proposed to address this limitation, but, typically, they slowly adapt to distribution shifts because they update predictions only in a forward manner, that is, they generate a prediction for a newly observed data point while previously computed predictions are not updated. In this paper, we propose a novel online conformal inference method with retrospective adjustment, which is designed to achieve faster adaptation to distributional shifts. Our method leverages regression approaches with efficient leave-one-out update formulas to retroactively adjust past predictions when new data arrive, thereby aligning the entire set of predictions with the most recent data distribution. Through extensive numerical studies performed on both synthetic and real-world data sets, we show that the proposed approach achieves coverage close to the nominal level while reducing predictive interval width by up to approximately 30\% compared to existing online conformal prediction methods, demonstrating improved statistical efficiency alongside faster adaptation.
\end{abstract}


\section{Introduction}

\label{sec:intro}

In this paper, we consider a task of quantifying the uncertainty around prediction in an online learning setup, where our aim is, for each time $t=1,2,\dots$,  to construct a prediction set for the target output $Y_{t+1}\in\R$ associated with a feature vector $X_{t+1}\in\R^d$ by using the information of the previously observed data $\cD_t:=\{(X_i, Y_i)\}_{i=1,\dots, t}$.  Specifically, for a specified target miscoverage level $\alpha\in(0,1)$, we wish to construct a set-valued statistic $\hat{C}_{t+1}:\R^d\mapsto 2^{\R}$ depending on $\cD_t$ and $X_{t+1}$, which guarantees $\P(Y_{t+1}\in\widehat{C}_{t+1})\ge1-\alpha$. The set $\hat{C}_{t+1}$ is  referred to as a $1-\alpha$ prediction set for $Y_{t+1}$. As the risk of incorrect predictions can be substantial in modern machine learning applications, it is essential to provide valid and well-calibrated prediction sets to enable more robust and reliable decision-making.

Conformal inference has gained its popularity for the construction of prediction sets, particularly due to its generality and broad applicability. Assuming only the exchangeability of the data, this offers a versatile framework that 
converts the outputs from any black-box prediction algorithm into a valid prediction set \citep{vovk2005algorithmic,shafer2008tutorial,lei2018distribution}. 
However, applying conformal inference methods in an online learning setup presents significant challenges, as the exchangeability assumption on the data often fails in practice. Non-stationary time series data serve as illustrative examples that are frequently observed in both natural phenomena and economic contexts \citep{raza2015ewma,Liu2018Acc}. Distribution shift is also common in modern data analysis \citep{zhang2025one, zhang2026one, wu2021top}, for instance, a credit scoring model trained on data from an older population may perform poorly when applied to a younger demographic, due to shifts in underlying data distribution \citep{hand1997statistical}. When exchangeability no longer holds, standard conformal inference methods may fail to achieve the nominal coverage level \citep{barber2023conformal,gibbs2021aci}.  To tackle this problem, a number of approaches have been proposed to extend conformal prediction to non-exchangeable and/or distribution-shifted data sets. Representative directions include covariate shift \citep{tibshirani2019conformal}, robust methods for distribution shift \citep{chernozhukov2018exact, yang2024doubly}, and online conformal inference for drifting streams \citep{gibbs2021aci}.

Nevertheless, most existing approaches remain \emph{slowly adaptive} to distribution shifts due to their training or updating schemes. Typically, these methods compute prediction intervals sequentially: when a new data point $(X_t, Y_t)$ arrives, a prediction for $Y_{t+1}$ is generated based on the current model trained on $\mathcal{D}_t$, while the previously computed predictions for $Y_1,\dots,Y_t$ remain fixed. As a result, when the data distribution evolves over time, these static historical predictions may become inconsistent with the current data-generating process, leading to a delayed adaptation to distributional changes.

In this paper, we propose a novel online conformal prediction method that addresses the challenge of slow adaptation under distribution shift. Our approach is based on a retrospective adjustment mechanism that dynamically revises past predictions as new data arrive. Specifically, leveraging regression estimators that admit closed-form leave-one-out (LOO) updates, our method updates the predicted values for previous outputs $Y_1, \dots, Y_t$ in an efficient manner, thereby improving the consistency of the residuals with the most recent data distribution. This \emph{retrospective adjustment} enables substantially faster adaptation compared to existing forward-update methods. In this work, we interpret faster adaptation as the ability to restore empirical coverage to the target level shortly after a shift occurs. This can be quantified in terms of the recovery time, defined as the number of time steps required for the coverage probability to return to its nominal level following a distributional change.

From a broader perspective, the proposed method is designed to address two fundamental challenges in online conformal inference. First, from a statistical viewpoint, existing approaches suffer from slow adaptation because they do not update previously computed residuals. Second, from a computational viewpoint, naively implementing retrospective updates would require repeated retraining, resulting in prohibitive $O(n^4)$ complexity, where $n$ denotes the training sample size. Our aim is to achieve the statistical benefits of retrospective recalibration, while avoiding this computational burden.

To this end, we restrict our attention to regression estimators that admit efficient LOO updates, such as linear smoothers including kernel ridge regression (KRR), instead of adopting a fully model-agnostic framework. While this limits the generality of the framework, this restriction is essential for making retrospective adjustment computationally tractable. Without such structure, retrospective updates would require repeated model refitting and become infeasible in online settings. At the same time, the class of linear smoothers remains sufficiently broad, encompassing widely used models such as KRR and neural tangent kernel methods.

The main contributions of this paper are summarized as follows:
\begin{itemize}
    \item \textbf{Retrospective adjustment for online conformal inference.} 
    We introduce a novel retrospective adjustment mechanism that updates the entire set of calibration residuals as new data arrive, enabling faster and more stable adaptation to distributional shifts.

    \item \textbf{Efficient $O(n^2)$ implementation.} 
    We develop an efficient algorithm based on rank-one matrix updates for kernel ridge regression, reducing the computational complexity of retrospective adjustment from $O(n^4)$ to $O(n^2)$ per time step.

    \item \textbf{Theoretical guarantee of long-run coverage.} 
     We establish that the proposed method achieves long-run calibration of the miscoverage rate by leveraging the ACI framework, without requiring exchangeability.
\end{itemize}

The rest of this paper is organized as follows. In \cref{sec:preliminary}, we briefly provide some background materials necessary to introduce our method. In \cref{sec:method}, we explain the proposed methodology and its advantages over the existing approaches. In \cref{sec:simulation,sec:realdata}, we perform numerical experiments on synthetic and real data sets, respectively, which demonstrate the superior performance of the proposed method.  \cref{sec:conclusion} concludes our paper.

\section{Preliminaries}\label{sec:preliminary}

We first introduce several notations. For a natural number $n$, we let $[n]:=\{1,\dots, n\}$. Let $I$ denote an identity matrix. For a set $A$, $|A|$ denotes its cardinality. For a finite set $A$ of real numbers, we let $\quantile_{\gamma}(A)$ denote the $\gamma$-th empirical quantile of the set $A$ for $\gamma\in[0,1]$. For a real set $A\subset \R$, we write its diameter as $\text{diam}(A):=\sup_{x,y\in  A}|x-y|$. Note that the diameter of a real interval is equal to its width.

\subsection{Jackknife+ Conformal Prediction}

Two standard conformal prediction approaches for exchangeable data are split and full conformal prediction. Split conformal prediction first partitions the whole data into a training set used to fit a prediction model, and a calibration set used to compute  conformity scores. This is computationally efficient but sacrifices the statistical efficiency since part of the data are allocated to the calibration set. In contrast, full conformal prediction does not hold out any data points for calibration, thereby not losing predictive accuracy. But full conformal prediction usually requires training a prediction model many times, which is computationally intensive.

Comparison of the full and split conformal methods highlights a trade-off between computational and statistical efficiency. The Jackknife+, originally proposed by \citet{barber2021predictive}, provides a compromise between these two extremes. Although Jackknife+ requires retraining the model $n$ times, it avoids the loss of a full calibration set. This approach starts with computing the  fitted regression function $\hat{f}_{[n]\setminus\{i\}}$ on the leave-one-out (LOO) sample $\{ (X_i, Y_i) : i \in[n]\setminus\{i\} \}$ and then computes the LOO residual $R_i := |Y_i - \hat{f}_{[n]\setminus\{i\}}(X_i)|$. The Jackknife+ prediction interval is given by
\begin{align*}
   \widehat{C}^{\textup{Jack+}}_{n+1}(\alpha) := 
    \Big[&-\quantile_{(1-\alpha)(1+1/n)}\del[1]{\{-\hat{f}_{[n]\setminus\{i\}}(X_{n+1})+R_i\}_{i\in[n]}},\\
    &\quantile_{(1-\alpha)(1+1/n)}\del[1]{\{\hat{f}_{[n]\setminus\{i\}}(X_{n+1})+R_i\}_{i\in[n]}}\Big]
\end{align*}
Assuming the exchangeability of the data, this interval satisfies $ \mathbb{P}\del[0]{Y_{n+1} \in\widehat{C}^{\textup{Jack+}}_\alpha(X_{n+1})}\geq 1 - 2\alpha.$
Although it only guarantees $1-2\alpha$ coverage theoretically, it was shown in \citet{barber2021predictive} that this achieves the target coverage $1-\alpha$ if the regression algorithm is stable. 

\subsection{Adaptive Conformal Inference}
\label{subsec:aci}

In this subsection, we briefly describe the adaptive conformal inference (ACI) method proposed by \citet{gibbs2021aci}. Let $\hat{f}_{[t-1]}$ be the fitted regression function  at time $t$, trained on the past observations $ \{(X_i, Y_i)\}_{i\in[t-1]}$ and $\mathcal{E}_t$ be a set of residuals for calibration. Then, the prediction interval for $Y_t$ is constructed as
\begin{align*}
   \widehat{C}^{\mathrm{ACI}}_{t}(\alpha_t) 
    :=\hat{f}_{[t-1]}(X_t) \pm   \quantile_{1-\alpha_t}(\mathcal{E}_{t}), 
\end{align*}
where $\alpha_t$ is the miscoverage level at time $t$, which will be adaptively updated after observing $Y_t$ as
\begin{align*}
    \alpha_{t+1} = \alpha_t + \gamma \big\{\alpha -  \ind(Y_t \notin\widehat{C}^{\mathrm{ACI}}_{t}(\alpha_t)) \big\},
\end{align*}
with $\gamma>0$ referred to as a \emph{step size} parameter. \citet{gibbs2021aci} showed that this updating scheme ensures that the long-run coverage converges to the target coverage $1-\alpha$ even when the data distribution evolves over time.

Even though theoretically justified, in practice, the performance of ACI is highly sensitive to the choice of the step size $\gamma$. If it is chosen too small, the intervals adapt slowly to changes in the distribution. If it is too large, the procedure may oscillate and produce unstable coverage. Several approaches have been proposed to address this issue, including Online Expert Aggregated ACI (AgACI, \citep{zaffran2022adaptive}),  Dynamically-Tuned Adaptive Conformal Inference (DtACI, \citep{gibbs2024dtaci}), Scale-Free Online Gradient Descent (SFOGD, \citep{bhatnagar2023improved}) and Strongly Adaptive Online Conformal Prediction (SAOCP,  \citep{bhatnagar2023improved}). For the reader's convenience, we provide the complete descriptions of these algorithms in \cref{appendix:aci_algorithms}.

\subsection{Regression with Linear Smoother}
\label{subsec:linear_smoother}

\paragraph{Linear smoother}
For notational simplicity, we write $ X_{1:n}:=(X_1,\dots, X_n)^\top \in\R^{n\times d}$ and $Y_{1:n}:=(Y_1,\dots, Y_n)^\top\in\R^n.$   A regression estimator $\hat{f}$ is called a linear smoother with smoothing function $\xi_n:\R^d\times (\R^d)^{n}\mapsto \R^n$ if it is given by, for every $x \in \R^d$,
\begin{align*}
    \hat{f}(x) = \xi_n(x, X_{1:n})^\top Y_{1:n}.
\end{align*}
Note that the smoothing function $\xi_n(x, X_{1:n})$ depends only on the features $X_{1:n}$ and the input $x$, but not on the responses $Y_{1:n}$. We call an $n\times n$ matrix $S$ with the $i$-th row vector being $(S_{i1},\dots, S_{in})^\top=\xi_n(X_i, X_{1:n})$ the smoother matrix associated with $\hat{f}$.

\paragraph{Kernel ridge regression} A widely used regression method leading to a linear smoother is  kernel ridge regression (KRR). For a positive-definite kernel function $\kappa:\R^d\times \R^d\mapsto [0,\infty)$ and a regularization parameter $\lambda>0$, the fitted KRR function on the data $\{(X_i,Y_i)\}_{i\in[n]}$ is given by
\begin{align*}
    \hat{f}_{[n]}(x) = \kappa(x, X_{1:n})(\kappa(X_{1:n}, X_{1:n}) + \lambda I)^{-1} Y_{1:n},
\end{align*}
where we denote $ \kappa(x, X_{1:n}):=(\kappa(x, X_i))_{i\in[n]}$ and $ \kappa(X_{1:n}, X_{1:n}):=(\kappa(X_i,X_j))_{i,j\in[n]}$. This formulation also encompasses neural tangent kernel (NTK) regression \citep{jacot2018neural}, which can be regarded as a special case of KRR using the NTK of a given neural network architecture.

\paragraph{Leave-one-out formula}  Let $\hat{f}_{[n]}$ be a linear smoother fitted on the whole data $\{(X_i, Y_i)\}_{i\in[n]}$ and $S =(S_{ij})_{i\in[n],j\in[n]}$ be its smoother matrix. 
We say that the linear smoother $\hat{f}_{[n]}$ is \emph{self-stable} if the fitted function based on the ``augmented'' data set $\{(X_i, Y_i)\}_{i\in[n]}\cup\{x,\hat{f}_{[n]}(x)\}$ is identical to the original one $\hat{f}_{[n]}$. It is easy to check that the KRR estimator is self-stable. For a linear smoother with the self-stable property, its LOO residuals can be expressed in a simple closed form without refitting the linear smoother.  Let  $\hat{f}_{[n]\setminus\{i\}}$ be a fitted one on the LOO data $\{(X_i, Y_i)\}_{i\in[n]\setminus\{i\}}$ removing the $i$-th observation. 

\begin{lemma}
\label[lemma]{lemma:looresid}
For a linear smoother with the self-stable property, the LOO residuals are given by
\begin{align*}
    Y_i - \hat{f}_{ [n] \setminus \{i\}}(X_i) 
    = \frac{Y_i - \hat{f}_{ [n] }(X_i)}{\,1 - S_{ii}} ,\quad i \in [n].
\end{align*}
\end{lemma}

\begin{proof}
The proof can be found in Theorem 2.7 of \citet{fan2020statistical}.
\end{proof}

Moreover, every prediction by the LOO linear smoother  $\hat{f}_{[n]\setminus\{i\}}$ can be computed by $\hat{f}_{[n]}$ without refitting.

\begin{lemma}
\label[lemma]{lemma:loopred}
For a linear smoother with the self-stable property, we have
\begin{align*}
    \hat{f}_{[n]\setminus \{i\}}(x)
    = \hat{f}_{ [n]}(x)
      - \frac{\, \xi_{n}^{i}(x) \,}{\,1 - S_{ii}\,}
        \bigl(Y_i - \hat{f}_{[n]}(X_i) \bigr),
\end{align*}
for all $x$, where $\xi_{n}^{i}(x)$ is the $i$-th element of the smoothing vector $\xi_n(x, X_{1:n})$.
\end{lemma}
\begin{proof}
The proof is deferred to \cref{proof:loopred}.
\end{proof}

\section{Online Conformal Inference with Retrospective Adjustment}\label{sec:method}

In this section, we present an efficient algorithm for online conformal prediction with  \emph{retrospective adjustment (RetroAdj)}.  Existing ACI-based methodologies introduced in \cref{subsec:aci}  primarily focus on updating the miscoverage level $\alpha_t$ in a statistically efficient manner,  while leaving the calibration set of residuals $\mathcal{E}_{t}$ unchanged. In other words, the previously computed residuals are typically not revised  after a new data point is observed. This lack of retrospective updating causes slow adaptation to distributional shifts, since the calibration set may be inconsistent with the most recent data distribution.

We first formalize the problem setup and introduce additional notation. We assume that the first $t_{\mathrm{init}}\in\bN$ sample points $\{(X_i, Y_i)\}_{i\in[t_{\mathrm{init}}]}$ are provided as initial data. At each subsequent time  $t >\tinit$, a new data point $(X_t, Y_t)$ arrives, and our goal is to construct a well-calibrated prediction set for $Y_t$ based on the previously observed data $\{(X_i, Y_i)\}_{i\in[t-1]}$.  In practice, it is often beneficial to restrict our attention to a subset of recent observations  rather than using all the past observations, since very old data points may originate from a substantially different data distribution. Motivated by this, we introduce a hyperparameter $w\in\bN$ that specifies the size of a ``time-window''. In other words, we only use the most recent $w$ data points for constructing the prediction set at time $t$. Formally, we define the index set of these ``active'' data points as
  \begin{align}
  \label{eq:training_data_indices}
        \cI(t):=\{\max\{t-1-w,1\},\max\{t-1-w,1\}+1,\dots, t-1\}\subset[t-1]
    \end{align}
The convention $w = \infty$ implies that all previously observed data are utilized at $t$.

\subsection{Jackknife+ with Efficient Leave-One-Out Formula}

The proposed online conformal inference procedure builds upon the Adaptive Conformal Inference (ACI) framework introduced in \cref{subsec:aci}, which aims to maintain the target long-run coverage level of $1-\alpha$. That is, at each time $t$, the miscoverage rate $\alpha_t$ is updated according to an ACI-type rule. However, unlike these existing approaches that update the calibration set by simply appending the most recent residual, our method employs  the Jackknife+ framework of \citet{barber2021predictive}, where we reconstruct \emph{all} residuals in the calibration set whenever a new data point arrives. This retrospective adjustment ensures faster adaptation to distributional shifts. 

However, a direct application of Jackknife+ requires $n_{t}$-many leave-one-out (LOO) residuals, where $n_{t}=|\cI(t)|$ denotes the number of active observations at time $t$. Naively refitting $n_t$ regression functions would be computationally prohibitive. Fortunately, when the underlying regression estimator is a linear smoother, this limitation can be overcome, since it allows an efficient LOO formula, as discussed in \cref{subsec:linear_smoother}. For simplicity of exposition, we focus on the case of kernel ridge regression (KRR), although the same principle applies to other linear smoothers. 
 
At each time $t$ after $\tinit$, we construct a prediction set for $Y_{t}$ as follows. Let $\hat{f}_{\cI(t)}$ be the fitted KRR function with kernel $\kappa$ on the data $\{(X_i, Y_i)\}_{i\in \cI(t)}$,  where $\cI(t)$ is defined in \labelcref{eq:training_data_indices}. Note that it is a linear smoother with smoothing function $\xi_n(x, (X_{i})_{i\in \cI(t)})= \{\uk^{(t)}(x)(K^{(t)}+ \lambda I)^{-1}\}^
\top$, where we define
    \begin{align*}
       \uk^{(t)}(x) :=(\kappa(x,X_i))_{i\in \cI(t)}, \quad
        K^{(t)}&:=(\kappa(X_i,X_j))_{i,j\in \cI(t)}.
    \end{align*}
Let $S^{(t)}:=K^{(t)}(K^{(t)}+ \lambda I)^{-1} $ be the corresponding smoother matrix. Then by \cref{lemma:looresid}, for each  $i \in \cI(t)$, the $i$-th LOO residual is given by
    \begin{align*}
        R_t^i := \bigl| Y_i - \hat{f}_{\cI(t) \setminus \{ i\}} (X_i) \bigr|=
        \left| \frac{\,Y_i - \hat{f}_{\cI(t)}(X_i)\,}{\,1 - S_{ii}^{(t)}\,} \right|,
    \end{align*}
where $S_{ii}^{(t)}$ denotes the $i$-th diagonal element of the smoother matrix $S^{(t)}$. 
Moreover, by \cref{lemma:loopred}, the  LOO predictions for a new response $Y_{t}$ can be computed as
\begin{align*}
    \hat{f}_{\mathcal{I}(t) \setminus \{ i\}}(X_{t}) 
    = \hat{f}_{\mathcal{I}(t)}(X_t) 
      - \frac{\xi_{n}^{i}(X_t)}{\,1 - S_{ii}^{(t)}\,}\bigl(Y_i - \hat{f}_{\mathcal{I}(t)}(X_i)\bigr),
\end{align*}
where $\xi_{n}^{i}(X_t)$ denotes the $i$-th element of the vector $\xi_n(X_t, (X_i)_{i\in\cI(t)})=\{  \uk^{(t)}(X_t)(K^{(t)}+ \lambda I)^{-1}\}^\top$. Using these quantities, we construct a Jackknife+ prediction interval 
with target coverage level $1 - \alpha_t$, which can be computed without refitting the KRR function for LOO datasets due to \cref{lemma:looresid,lemma:loopred} as
\begin{align}
   \widehat{C}_t^{\text{RA}}(\alpha_t)
    := \big[&-\quantile_{(1-\alpha_t)(1+1/n_t)}\del[1]{\{-\hat{f}_{\cI(t) \setminus \{ i\}}(X_t) + R_t^i\}_{i\in \mathcal{I}(t)}},\nonumber\\
    &     \quantile_{(1-\alpha_t)(1+1/n_t)}\del[1]{\{\hat{f}_{\cI(t) \setminus \{ i\}}(X_t) + R_t^i\}_{i\in \mathcal{I}(t)}}\big].
    \label{jackknife_interval}
\end{align}
As $\alpha_t$ can be less than 0 or greater than $1$ during an ACI procedure, for completeness, we let $\widehat{C}_t^{\text{RA}}(\alpha_t)=\emptyset$ if $\alpha_t<0$ and let $\widehat{C}_t^{\text{RA}}(\alpha_t)=\R$ if $\alpha_t>1$. The superscript RA stands for retrospective adjustment, emphasizing that our method updates all residuals retrospectively rather than adding the currently computed residual incrementally. In the next subsection, we give a detailed explanation on this property.

\subsection{Retrospective Adjustment}

In what follows, we refer to $\hat{f}_{\mathcal{I}(t)}$ as the \emph{base estimator}, since it serves as the computational basis from which all LOO residuals and predictions can be derived efficiently as discussed in the previous subsection. Nevertheless, to construct the prediction interval $\widehat{C}_t^{\mathrm{RA}}$ in the online learning setup, it is necessary to compute the base estimator $\hat{f}_{\mathcal{I}(t)}(x)$ at each time step $t$. A naive implementation of KRR requires recomputing the matrix inverse $Q^{(t)}:=(K^{(t)} + \lambda I)^{-1}$ from scratch whenever a new data point arrives. This leads to a computational cost of $O(n_t^3)$ per update, which becomes infeasible for large-scale applications.

To overcome this limitation, we leverage the block matrix inversion \citep{lu2002inverses} to update the inverse $Q^{(t)}$ efficiently.  If the current time step $t$ is less than or equal to the specified window size $w$, the oldest observation $(X_{t-1-w},Y_{t-1-w})$  will not be discarded, i.e., $\cI(t)\setminus\{t-1-w\}=\cI(t)$.  Otherwise, as the oldest observation $(X_{t-1-w},Y_{t-1-w})$ is removed from the set $\cI(t)$ of active observations, we first remove its contribution using a symmetric rank-one ``downdate'' formula given in the next lemma.

\begin{lemma}
\label[lemma]{lemma:krr_downdate}
Suppose that $t>w$. Consider the following partition of the matrix $Q^{(t)}$:
\begin{align*}
    Q^{(t)} = \begin{pmatrix}
       q_{11} & q_{12}^\top \\
        q_{12} & Q_{22}
    \end{pmatrix},
\end{align*}
with $q_{11} \in \mathbb{R}$, $q_{12} \in \mathbb{R}^{n_t-1}$, and $Q_{22} \in \mathbb{R}^{(n_t-1)\times (n_t-1)}$. Then the inverse of the matrix $\check{K}^{(t)}+ \lambda I$ with $\check{K}^{(t)}:=(\kappa(X_i,X_j))_{i,j\in  \cI(t)\setminus\{t-1-w\}}$ can be computed as
\begin{align}
    \label{eq:inverse_downdate}
  (\check{K}^{(t)}+ \lambda I)^{-1}= Q_{22} - \frac{1}{q_{11}}q_{12} q_{12}^\top.
\end{align}
\end{lemma}

\begin{proof}
The proof is deferred to \cref{proof:krr_downdate}.
\end{proof}

Now, let $\check{Q}^{(t)}$ be a matrix equal to $Q^{(t)}$ when $t\le w$ and equal to $ (\check{K}^{(t)}+ \lambda I)^{-1}$ in \labelcref{eq:inverse_downdate} when $t>w$. Next, we compute the matrix inverse $Q^{(t+1)}:=(K^{(t+1)} + \lambda I)^{-1}$ with $K^{(t+1)}=(\kappa(X_i,X_j))_{i,j\in  \cI(t+1)}$ for the next time step, which reflects the information of the newly arrived observation $(X_t, Y_t)$. This computation can be performed efficiently, employing the rank-one correction given in the following lemma.

\begin{lemma}
\label[lemma]{lemma:krr_update}
For simplicity, we denote by $\check{Q}=\check{Q}^{(t)}$, $ \uk:=(\kappa(X_{t+1},X_{i}))_{i\in \cI(t)\setminus\{t-1-w\}} $ and $\delta=(1 + \lambda -  \uk^\top \check{Q}  \uk)^{-1}$.
Then the matrix inverse $Q^{(t+1)}$ can be computed as
\begin{align}
   Q^{(t+1)}= (K^{(t+1)}+ \lambda I)^{-1}=
    \begin{pmatrix}
        \check{Q} + \delta (\check{Q} \uk)(\check{Q} \uk)^\top & -\delta \check{Q}  \uk \\
        -\delta (\check{Q}  \uk)^\top & \delta
    \end{pmatrix}.
\end{align}
\end{lemma}
\begin{proof}
The proof is deferred to \cref{proof:krr_update}.
\end{proof}

An important consequence of this efficient matrix update is that \emph{all} LOO residuals and predictions can be \emph{revised} at every time step without a heavy computational burden. That is, when we construct the proposed prediction set for a new response $Y_{t+1}$ at the next time step $t+1$, the base estimator $\hat{f}_{\mathcal{I}(t+1)}$ can be efficiently updated from the previous one $\hat{f}_{\mathcal{I}(t)}$ with the newly observed data point $(X_t,Y_t)$, and accordingly, all LOO residuals 
$R_{t+1}^i:=|Y_i-\hat{f}_{\mathcal{I}(t+1)\setminus\{i\}}(X_i)|$ and predictions $\hat{f}_{\mathcal{I}(t+1)\setminus\{i\}}(X_{t+1})$ for $i\in\mathcal{I}(t+1)$ are simultaneously updated. This retrospective recalibration, which revises these ``past'' quantities using the current observation, reduces the discrepancy between the entire calibration set with the most recent data distribution. In principle, this allows the conformal prediction intervals to adapt promptly to distributional shifts.

We would like to clarify that the proposed method is not a simple combination of the Jackknife+ and ACI frameworks, but introduces a fundamentally different mechanism based on \emph{retrospective recalibration of residuals}. Existing ACI-based methods adapt the miscoverage level $\alpha_t$ over time but keep previously computed residuals fixed, while Jackknife+-based methods construct prediction sets using LOO residuals but do not incorporate any adaptive updating mechanism under distribution shift.  In contrast, our method introduces a new dimension by dynamically updating the entire set of residuals as new data arrive. This retrospective adjustment modifies the calibration set itself, rather than only adjusting the threshold (as in ACI) or relying on static residuals (as in Jackknife+). As a result, the proposed approach fundamentally differs from existing methods in how adaptation is achieved.

\subsection{Summary of the Proposed Algorithm}

The proposed RetroAdj procedure is summarized in \cref{alg:oj_krr}.

\begin{algorithm}[H]
\caption{Online conformal inference with retrospective adjustment (RetroAdj)}
\label{alg:oj_krr}
\begin{algorithmic}
\State \textbf{Input:} window size $w\in\bN\cup\{\infty\}$, target miscoverage level $\alpha\in(0,1)$, initial miscoverage level $\alpha_{\tinit+1}$, initial base estimator $\hat{f}_{\cI(\tinit+1)}$.
\For{$t =\tinit+1,\tinit+2, \ldots,T$}
  \State \texttt{\slash\slash Constructing a prediction set}
  \State Observe $X_t$.
  \State \textbf{Return} prediction interval $\hat{C}_t^{\text{RA}}(\alpha_t)$ given in \eqref{jackknife_interval}.
  \State \texttt{\slash\slash Updating the base estimator and miscoverage level}
  \State Observe $Y_t$.
  \State Update $\hat{f}_{\cI(t+1)}$ from $\hat{f}_{\cI(t)}$  using \cref{lemma:krr_update} and, if $t>w$, using \cref{lemma:krr_downdate} as well.
 \State Update $\alpha_{t+1}$ from $\alpha_t$ by performing one of ACI-based algorithms.
\EndFor
\end{algorithmic}
\end{algorithm}

The RetroAdj algorithm achieves substantial computational efficiency by leveraging 
\cref{lemma:looresid,lemma:loopred,lemma:krr_update,lemma:krr_downdate}.
A naive implementation would require reconstructing the base estimator
$\hat{f}_{\mathcal{I}(t)}$ from scratch at each time step and re-fitting the KRR model 
$n_t$ times in order to compute all LOO residuals and predictions. Since each re-fitting involves matrix inversion with computational cost of $O(n_t^3)$, 
this naive approach would result in an overall complexity of $O(n_t^4)$ per step, which might be impractical for a large number of active observations.
In contrast, our construction only requires a computational cost of $O(n_t^2)$ per step, which is a substantial improvement in scalability.

In \cref{tab:comparison}, we compare the proposed RetroAdj with three representative baselines: the conventional forward online conformal inference method,  the standard Jackknife+ approach and a naive online Jackknife+ approach, which recomputes LOO residuals at each time step without efficient updates. For a fair comparison, all computational complexities are reported under the KRR setting. The forward method and the standard Jackknife+ approach are implemented using efficient LOO updates, whereas the naive online Jackknife+ baseline incurs an $O(n_t^4)$ cost due to repeated model refitting. The table emphasizes that the proposed RetroAdj method is the only approach that simultaneously achieves retrospective recalibration, computational efficiency via rank-one updates, and long-run coverage guarantees.

\begin{table}[t]
\centering
\caption{Comparison of methods under the KRR setting in terms of residual updating, computational cost, and theoretical guarantees.}
\label{tab:comparison}
\small
\begin{tabular}{lccc}
\hline
Method & Residual Update & Complexity (per step) & Coverage Guarantee \\
\hline
Forward
& $\times$ 
& $O(n_t^2)$ 
& ACI (long-run) \\

Standard Jackknife+
& $\times$ 
& $O(n_t^2)$ 
& $1-2\alpha$ (exchangeable) \\

Naive Online Jackknife+ 
& $\checkmark$ 
& $O(n_t^4)$ 
& ACI (long-run) \\

RetroAdj (ours) 
& $\checkmark$ 
& $O(n_t^2)$ 
& ACI (long-run) \\
\hline
\end{tabular}
\end{table}

\subsection{Theoretical Guarantee for Long-Run Coverage}

Since our approach is built upon the ACI framework, it inherits its asymptotic coverage guarantees in the long run, which is stated in the next theorem.

\begin{theorem}
\label{thm:longterm_coverage}
When the miscoverage level $\alpha_t$ is updated at each time $t$  using either ACI (\cref{alg:aci}) or SFOGD (\cref{alg:sfogd})  with a fixed step size $\gamma>0$, then we have
    \begin{align*}
       \frac{1}{T-\tinit}\sum_{t=\tinit+1}^T   \ind(Y_t \notin\widehat{C}^{\mathrm{RA}}_{t}(\alpha_t))\to\alpha
    \end{align*}
as $T\to\infty$ with probability one. Moreover, when $\alpha_t$ is updated using either DtACI (\cref{alg:dtaci}) with $\eta_t,\sigma_t \to 0$ or SAOCP (\cref{alg:saocp}, assuming mild regularity conditions as detailed in \cref{proof:longterm_coverage}) with fixed $\gamma >0$, then we have
    \begin{align*}
       \frac{1}{T-\tinit}\sum_{t=\tinit+1}^T   \E[\ind(Y_t \notin\widehat{C}^{\mathrm{RA}}_{t}(\alpha_t))]\to\alpha
    \end{align*}
as $T\to\infty$ with probability one, where the expectation is taken over the algorithmic randomness in DtACI or SAOCP. %
\end{theorem}

\begin{proof}
The proof is deferred to \cref{proof:longterm_coverage}.
\end{proof}

\cref{thm:longterm_coverage} builds upon the long-run coverage guarantees of the ACI framework. Our key point is that the proposed retrospective adjustment mechanism modifies how the residuals and prediction sets are constructed, but does not alter the update rule of the miscoverage level $\alpha_t$, which is the central component driving the theoretical guarantee. Specifically, the convergence result relies on the update dynamics of $\alpha_t$, which depend only on the sequence of coverage indicators $\ind(Y_t \in \widehat{C}_t)$. While the retrospective adjustment affects the construction of $\widehat{C}_t$, it does not violate the conditions required for the ACI-based updates to ensure long-run calibration. In particular, the ACI framework does not require exchangeability, but only assumes that the prediction set sizes are appropriately regulated by the miscoverage level $\alpha_t$.  That is, as long as the prediction sets expand or contract in response to $\alpha_t$, the update dynamics can correctly drive the empirical miscoverage toward the target level over the ``long-run''.
Our method preserves this property, since the width of the Jackknife+ prediction set is directly controlled through $\alpha_t$ via the quantile-based construction, even under retrospective updates.

\section{Simulation Study}
\label{sec:simulation}

In this section, we conduct a numerical study to support the validity, efficiency, and adaptivity of the proposed methodology for online conformal inference. The source code is available at \url{https://github.com/jungbinary/OnlineConformalRetroAdj}.

\paragraph{Baselines} To evaluate the effectiveness of our proposed online conformal inference with retrospective adjustment (RetroAdj), we perform experiments comparing RetroAdj with conventional  ``forward'' online conformal inference methods (FW), which update the calibration set by incrementally adding a newly computed residual without retrospective adjustment. We consider four regression methods: Kernel Ridge Regression (KRR), the Fast Incremental Model Tree with Drift Detection (FIMT-DD, \cite{ikonomovska2011learning}) that incrementally updates leaf-level linear models and employs drift detectors to handle non-stationarity, the Adaptive Model Rules for Regression (AMRules, \cite{almeida2013adaptive}) that constructs a set of local linear models, each associated with an adaptive rule that covers a subregion of the input space, and the Aggregated Mondrian forests for Online Learning (AMF, \cite{mourtada2021amf}) that updates trees sequentially and aggregates predictions via exponential weighting. These algorithms are designed for evolving data streams and serve as strong nonparametric baselines. We test both the proposed and competing methods with five ACI-based algorithms (ACI, AgACI, DtACI, SFOGD, and SAOCP) for updating the parameter $\alpha_t$ over time. We provide details of the implementation in \cref{appendix:implementation}.

\paragraph{Implementation of RetroAdj}  The RBF kernel $  \kappa(x, y) = \exp \left( -\|x - y\|^2/(2\sigma^2)\right)$ is used for both RetroAdj and FW-KRR methods. The regularization parameter $\lambda$ and the kernel bandwidth $\sigma^2$  were selected via leave-one-out cross-validation on the initial dataset. We fix the window size $w=250$ throughout the experiments, which gives stable performance in most settings. But as shown in our experiment in \cref{appendix:window_size}, the choice of $w$ has a significant impact on the performance. In practice, observable quantities such as local coverage behavior and interval width may provide useful guidance for assessing the adequacy of the window size.  For example, overcoverage together with excessively wide intervals may indicate that $w$ is too large and contains outdated observations, whereas unstable coverage or highly variable interval widths may suggest that $w$ is too small. Incorporating such monitoring strategies into online window selection is an interesting direction for future work.

\paragraph{Synthetic data generation}
We consider the following two data-generating processes with abrupt distribution shifts. 
\begin{itemize}

    \item \textbf{Setting 1 (Linear Model)} 
    We generate $X_t \iidsim\mathcal{N}(0, I_{10})$ and $Y_t \indsim \mathcal{N}(X_i^\top \beta^{(t)} , \tfrac{1}{2})$, where the coefficient vector $\beta^{(t)}$ changes as
    \begin{align*}
    \beta^{(t)} &=
        \begin{cases}
            (1.0,\, 0.8,\, 0.0,\, 0.0,\, 0.5,\, 0.0,\, 0.3,\, 0.0,\, 0.0,\, 0.2), & 1 \leq t\le  250, \\
               (0.0,\,-1.2,\, 0.7,\, 0.4,\, 0.0,\, 0.0,\, 0.9,\, 0.0,\,-0.6,\, 0.0), & 251 \leq t \le 1000
        \end{cases}.
    \end{align*}

    \item \textbf{Setting 2 (Non-Linear Bump Model)} 
    We generate $X_t \iidsim\mathrm{Unif}([0,1]^3)$ and $ Y_t \indsim  \mathcal{N}(\varphi^{(t)}(X_t), \tfrac{1}{2}),$ where the regression function changes as
    \begin{align*}
        \varphi^{(t)}(x) = 
        \begin{cases}
             g(x; (0.25,0.25,0.25)), & 1 \le t \le 250, \\
           - g(x; (0.3,0.4,0.5)) + 0.4, &251 \le t \le 1000
        \end{cases}
    \end{align*}
   with $g(x,c)$ denoting the Wendland kernel \citep{wendland1995piecewise},
    \begin{align*}
        g(x; c) = (1-\lVert x-c\rVert_{2})_{+}^{6}\,
        (35\,\lVert x-c\rVert_{2}^{2}+18\,\lVert x-c\rVert_{2}+3).
    \end{align*}

\end{itemize}

\paragraph{Results} For each method, we compute the average of empirical coverage and prediction interval widths over all time steps and 50 simulation replications. All prediction intervals are constructed to achieve the target coverage level of $1-\alpha = 0.9$. In our experiments, empirical coverage is treated as the primary criterion, since the main objective of predictive inference is to attain the target coverage level. Prediction interval width is used as a secondary measure of statistical efficiency, and is compared mainly among methods whose coverage is close to the nominal target. Thus, our evaluation is not based on an unrestricted trade-off between coverage and width, but rather on a two-stage principle: validity first, efficiency second.

\cref{tab:syn_main_results} present the results for Settings 1 and 2, respectively. Across all tuning strategies of ACI, RetroAdj consistently attains the target coverage level, while maintaining the shortest prediction interval width. In contrast, the competing methods produce overly conservative prediction intervals that are substantially wider, indicating lower statistical efficiency. We observe that the improvements achieved by RetroAdj are consistently larger than the corresponding standard deviations, indicating that the differences are not attributable to random variation. We emphasize that comparing RetroAdj with FW-KRR same base learner isolates the contribution of the retrospective adjustment mechanism, which offers a clear advantage in terms of robustness to distribution shifts as well as improved statistical efficiency. Moreover, RetroAdj outperforms the other forward methods using more advanced adaptive models which have own built-in mechanisms for handling concept drift at the model level. This implies that the retrospective adjustment approach, which operates at the level of interval construction, can provide additional gains beyond model-level adaptation.

We additionally consider local performance measures to assess the adaptivity of  the proposed RetroAdj. Specifically, we consider the local average coverage rates and the local average of prediction interval widths 
    \begin{align*}
        \textup{LocCov}_t := \frac{1}{250} \sum_{s = t-250+1}^t \mathbb{I} \{ Y_s \in  \hat C_s \},\quad 
         \textup{LocWidth}_t := \frac{1}{250} \sum_{r = s-250+1}^t \text{diam}(\hat C_s ),
    \end{align*}
over a moving window of 250 time steps,   where $\hat{C}_s$ is a prediction interval at time $s$.  \cref{fig::fig_one_rep} shows that the local coverage of RetroAdj remains close to the target coverage even after the distribution shifts occuring at $t=251$, although the other methods struggle to adapt. Moreover, the interval width rapidly contracts over time. These results demonstrate that RetroAdj achieves strong adaptivity and stability in both settings.

\begin{table}[t!]
\centering
\caption{Comparison of empirical coverage and average interval width under the two synthetic settings. Values are reported as mean (standard deviation) over repeated runs. The target coverage is $1-\alpha=0.9$. The closest coverage to the target value and the shortest interval length are bolded.}
\label{tab:syn_main_results}
\setlength{\tabcolsep}{4pt}
\begin{tabular}{llcccc}
\toprule
& & \multicolumn{2}{c}{Setting 1} & \multicolumn{2}{c}{Setting 2} \\
\cmidrule(lr){3-4}\cmidrule(lr){5-6}
ACI & Model & Coverage & Width & Coverage & Width \\
\midrule
\multirow{5}{*}{ACI}
& FW-AMF     & 0.925 (0.008) & 5.285 (0.179) & 0.932 (0.009) & 2.536 (0.150) \\
& FW-AMRules & 0.958 (0.006) & 3.856 (0.178) & 0.916 (0.010) & 2.746 (0.157) \\
& FW-FIMTDD  & 0.914 (0.009) & 6.266 (0.250) & 0.915 (0.012) & 2.809 (0.182) \\
& FW-KRR     & 0.955 (0.006) & 4.132 (0.209) & 0.937 (0.008) & 2.629 (0.150) \\
& RetroAdj   & \textbf{0.905 (0.004)} & \textbf{2.887 (0.116)} &\textbf{ 0.903 (0.004)} & \textbf{2.134 (0.091)} \\
\midrule
\multirow{5}{*}{AgACI}
& FW-AMF     & 0.924 (0.007) & 5.289 (0.168) & 0.931 (0.009) & 2.538 (0.152) \\
& FW-AMRules & 0.956 (0.007) & 3.887 (0.201) & 0.915 (0.010) & 2.751 (0.158) \\
& FW-FIMTDD  & 0.913 (0.008) & 6.265 (0.246) & 0.913 (0.011) & 2.814 (0.178) \\
& FW-KRR     & 0.952 (0.006) & 4.160 (0.233) & 0.936 (0.008) & 2.638 (0.166) \\
& RetroAdj   & \textbf{0.905 (0.004) }& \textbf{2.941 (0.091)} & \textbf{0.903 (0.003)} & \textbf{2.199 (0.095)} \\
\midrule
\multirow{5}{*}{DtACI}
& FW-AMF     & 0.925 (0.007) & 5.290 (0.171) & 0.931 (0.009) & 2.544 (0.152) \\
& FW-AMRules & 0.956 (0.007) & 3.884 (0.189) & 0.915 (0.010) & 2.747 (0.152) \\
& FW-FIMTDD  & 0.913 (0.009) & 6.267 (0.236) & 0.914 (0.011) & 2.815 (0.179) \\
& FW-KRR     & 0.952 (0.006) & 4.163 (0.218) & 0.935 (0.008) & 2.643 (0.157) \\
& RetroAdj   & \textbf{0.904 (0.004)} & \textbf{2.918 (0.094)} & \textbf{0.901 (0.004)} & \textbf{2.153 (0.085)} \\
\midrule
\multirow{5}{*}{SAOCP}
& FW-AMF     &\textbf{0.905 (0.015)} & 4.928 (0.246) & 0.924 (0.014) & 2.504 (0.207) \\
& FW-AMRules & 0.949 (0.009) & 3.662 (0.236) & \textbf{0.905 (0.016)} & 2.664 (0.215) \\
& FW-FIMTDD  & 0.906 (0.014) & 6.153 (0.295) & \textbf{0.905 (0.020)} & 2.740 (0.245) \\
& FW-KRR     & 0.945 (0.009) & 3.872 (0.244) & 0.927 (0.013) & 2.572 (0.223) \\
& RetroAdj   & 0.913 (0.005) & \textbf{3.039 (0.109)} & 0.909 (0.006) & \textbf{2.272 (0.102)} \\
\midrule
\multirow{5}{*}{SFOGD}
& FW-AMF     & 0.911 (0.010) & 4.991 (0.178) & 0.923 (0.010) & 2.427 (0.137) \\
& FW-AMRules & 0.949 (0.007) & 3.540 (0.144) & 0.907 (0.012) & 2.632 (0.154) \\
& FW-FIMTDD  & 0.906 (0.010) & 6.092 (0.251) & 0.906 (0.015) & 2.710 (0.195) \\
& FW-KRR     & 0.946 (0.007) & 3.755 (0.169) & 0.926 (0.011) & 2.471 (0.142) \\
& RetroAdj   & \textbf{0.902 (0.006)} & \textbf{2.722 (0.091)} & \textbf{0.901 (0.007)} & \textbf{2.074 (0.076)} \\
\bottomrule
\end{tabular}
\end{table}

\begin{figure}[H]
    \centering
    \includegraphics[width=\textwidth]{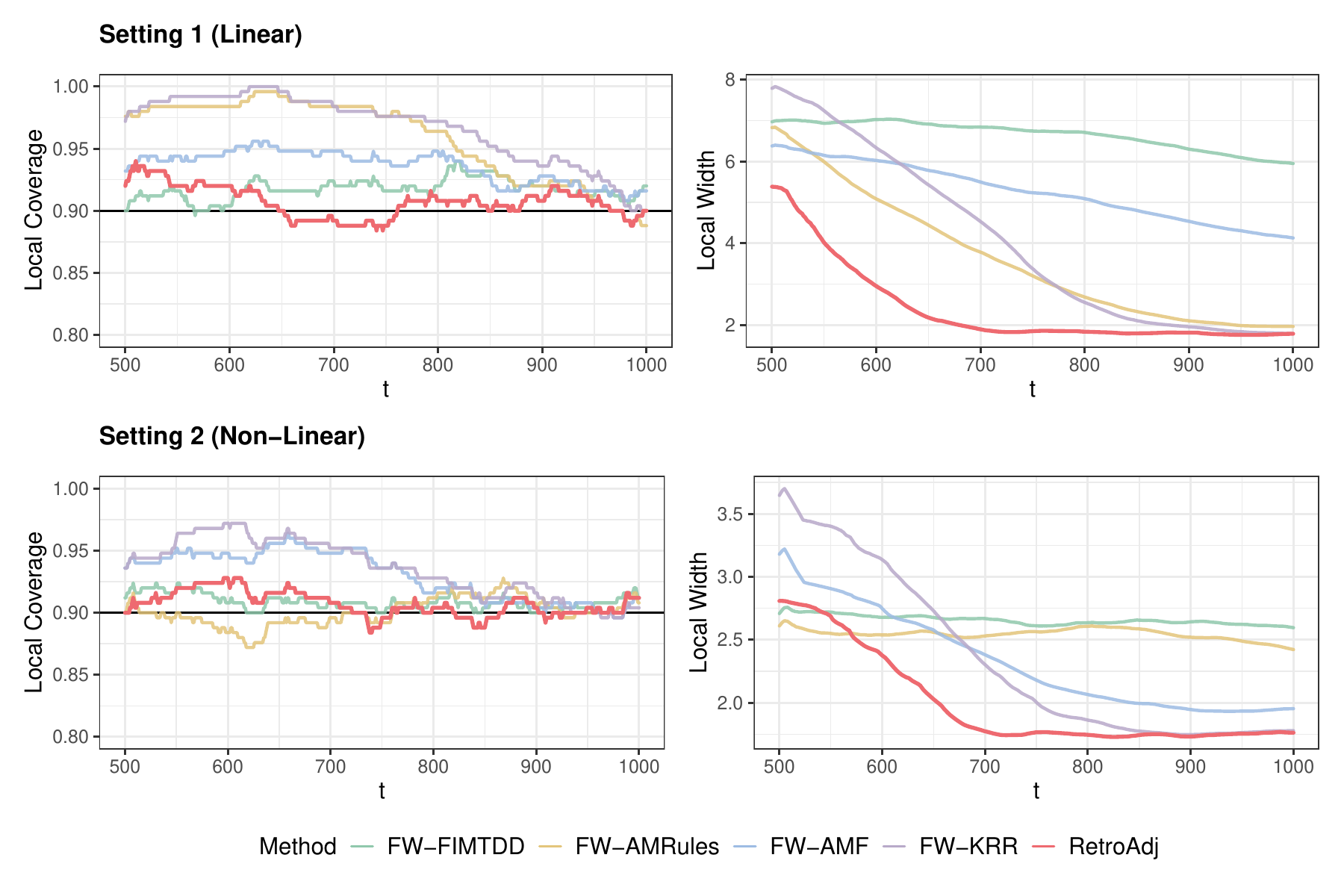}
    \caption{Local coverage and prediction interval width of the proposed RetroAdj and forward online conformal inference methods (FW) for Setting 1 and 2. For all methods, the DtACI algorithm is employed to adjust the miscoverage level.}
    \label{fig::fig_one_rep}
\end{figure}

\section{Real Data Analysis}
\label{sec:realdata}

\paragraph{Datasets}

We evaluate the proposed method on three real-world datasets that represent distinct types of distributional challenges in online prediction tasks. The Communities and Crime (CC) dataset \citep{redmond2002data} captures heterogeneous tabular data with complex feature interactions, providing a setting with potential covariate shifts. The Elec2 dataset \citep{harries1999splice} consists of electricity price time series with frequent and irregular fluctuations, reflecting gradual and recurrent distributional changes. 
The AIG stock price data \citep{nugent2018sp500} exhibit abrupt regime shifts associated with major economic events, representing a challenging scenario with sudden and large-scale distributional changes. 

For the CC data, we aim to predict the real-valued per-capita violent crime rate from 127 input features. We sort all observations in ascending order of the proportion of Black population and use the first 250 observations for training, while the remaining 1,774 samples are sorted in descending order and used as the test set. For the Elec2 data, we reconstruct the univariate time series into a dataset consisting of input-output pairs, where the output $Y_t$ is the current electricity price and the input $X_t = (Y_{t-1}, \dots, Y_{t-10})$ contains the past ten lagged values. For the AIG stock price data,  we use 3,000 observations before and after  the Subprime Mortgage Crisis (2008-09-15) each, during which the stock price exhibits a drastic regime change: its trend reverses sharply and its scale collapses to a much smaller magnitude as shown in \cref{fig:aig}. As we did for Elec2 data, we reconstruct the univariate time series into a dataset consisting of 10-day lagged input-output pairs.

\begin{figure}[H]
    \centering
    \includegraphics[width=0.8\textwidth]{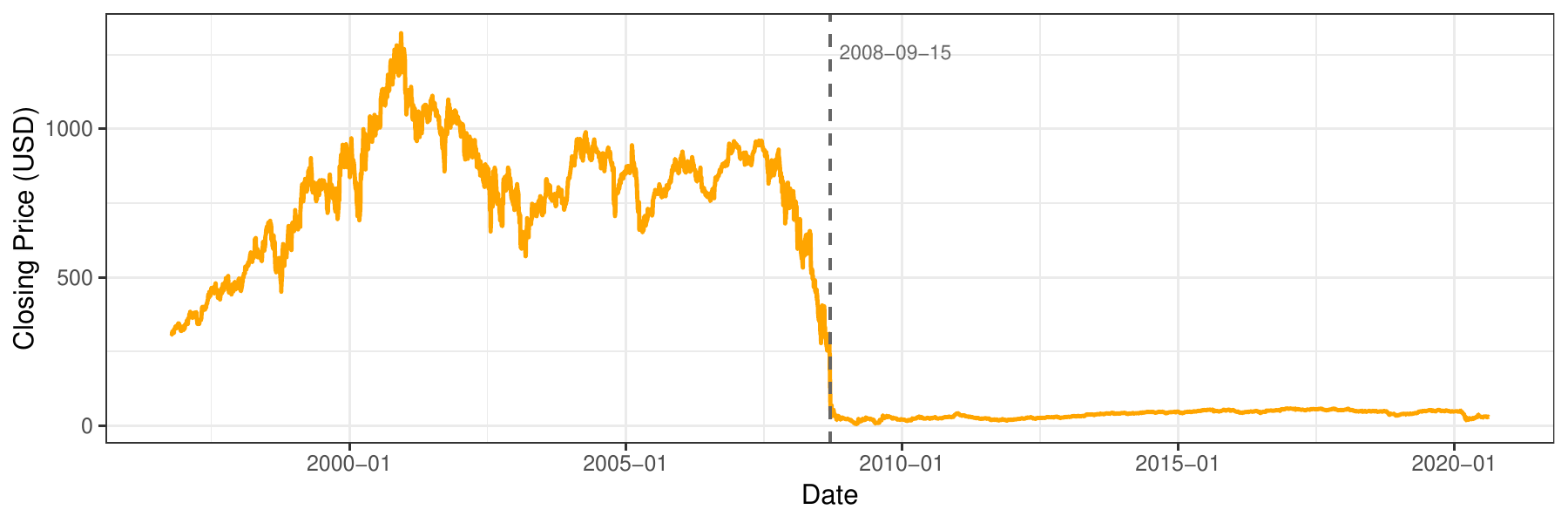}
    \caption{AIG Stock Price.}
    \label{fig:aig}
\end{figure}

\paragraph{Results} \cref{fig:real_world} shows that RetroAdj outperforms the baselines in both coverage and prediction interval width across all datasets. Notably, while all three baseline methods exhibit sharp drops in coverage and substantial increases in interval width when the data distribution changes, followed by gradual recovery. In contrast, RetroAdj maintains a nearly constant coverage level throughout the time horizon, demonstrating strong robustness to distribution shifts with almost no additional cost in predictive interval width.

\begin{figure}[t!]
    \centering
    \begin{subfigure}{\textwidth}
        \includegraphics[width=\textwidth]{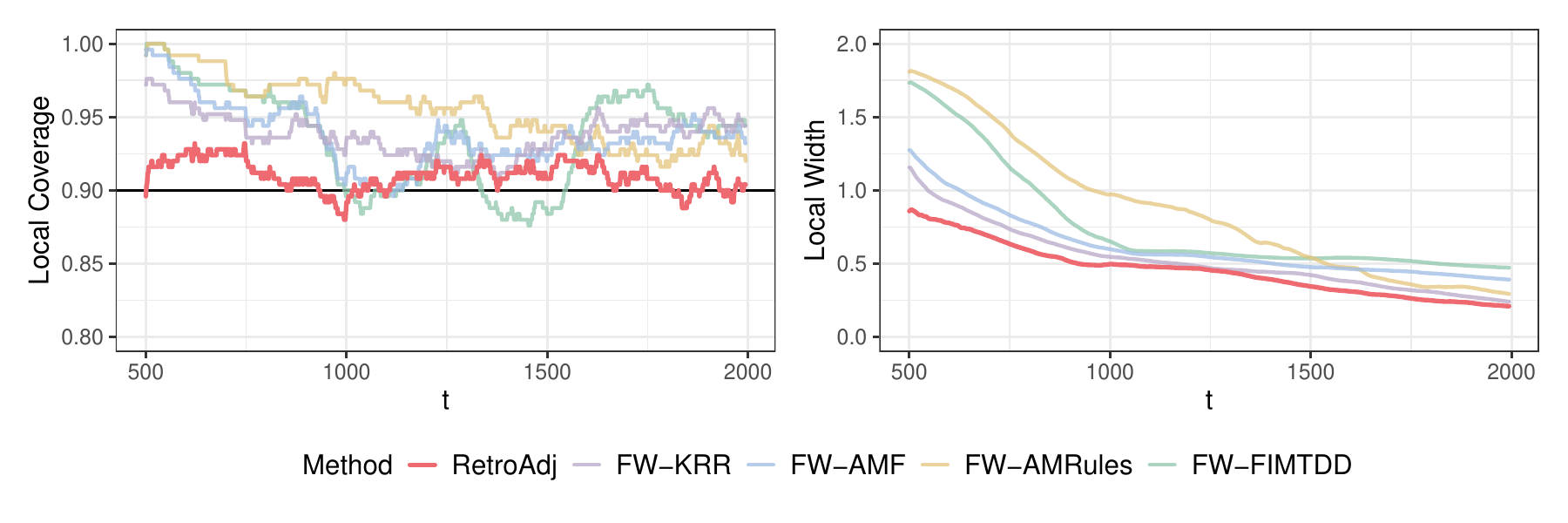}
        \caption{Community and Crime}
    \end{subfigure}\\
    \begin{subfigure}{\textwidth}
        \includegraphics[width=\textwidth]{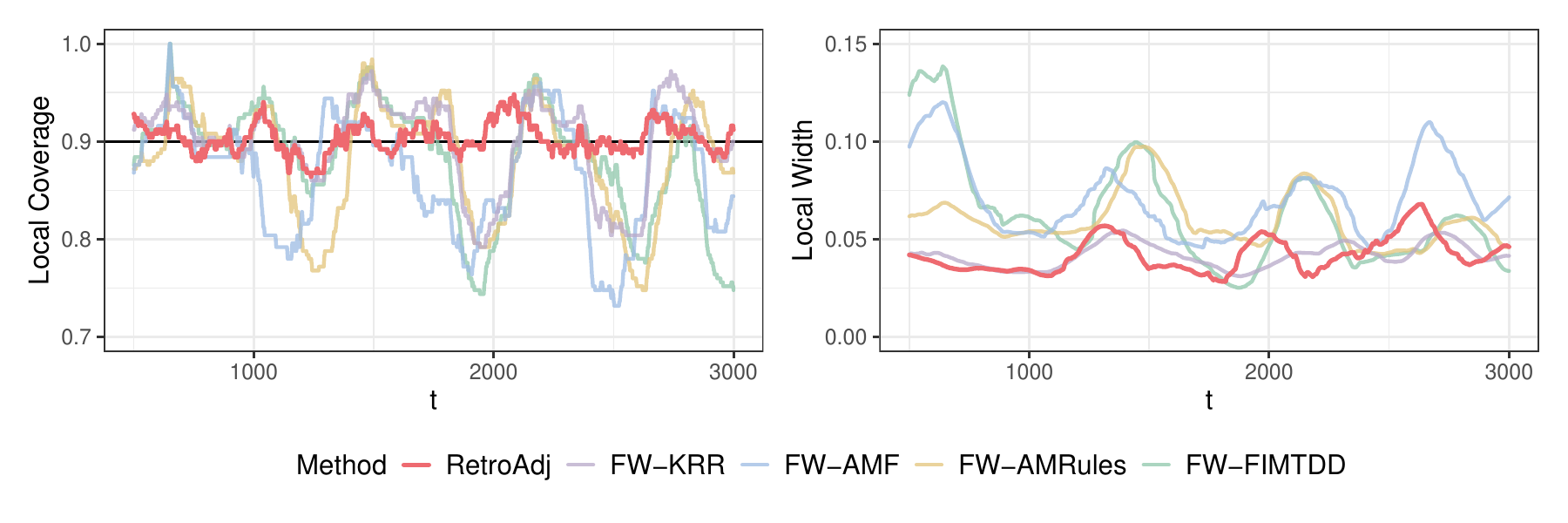}
        \caption{Elec2}
    \end{subfigure}\\
    \begin{subfigure}{\textwidth}
        \includegraphics[width=\textwidth]{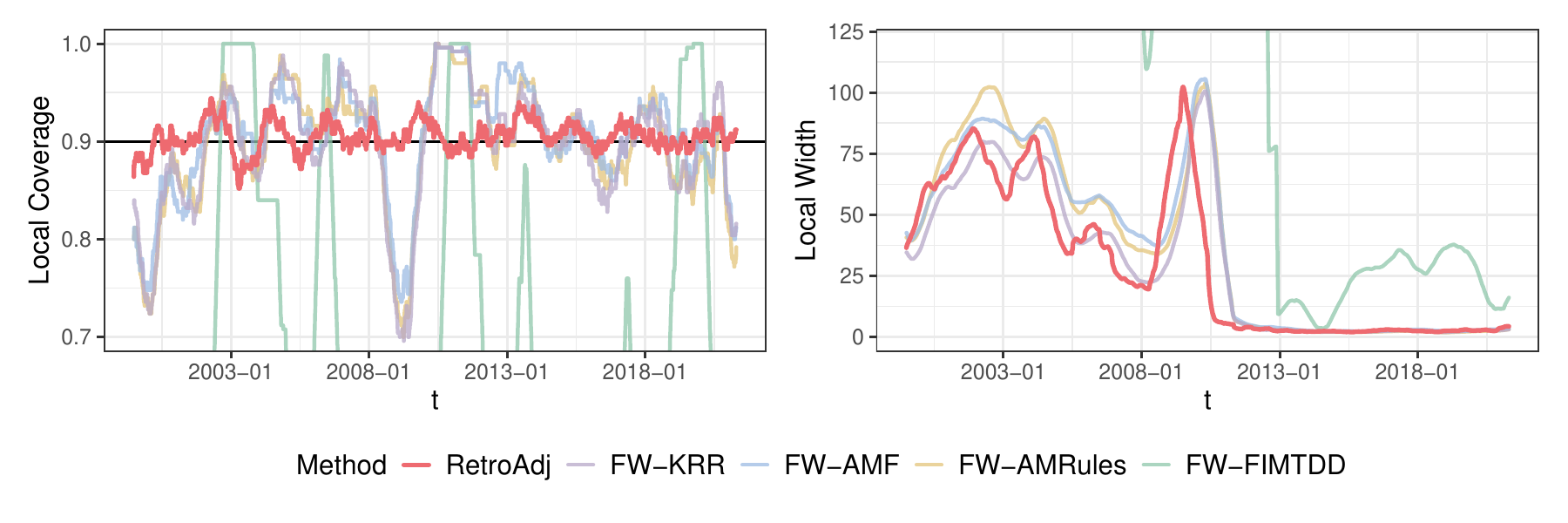}
        \caption{AIG stock price}
    \end{subfigure}
    \caption{Local coverage and prediction interval width of the proposed RetroAdj and forward online conformal inference methods (FW). For all methods, the DtACI algorithm is employed to adjust the miscoverage level.}
    \label{fig:real_world}
\end{figure}

\section{Conclusion}
\label{sec:conclusion}

In this work, we proposed an efficient framework for online conformal inference with \emph{retrospective adjustment}, designed to achieve faster adaptation for evolving data distributions over time.  By leveraging regression approaches with closed-form leave-one-out (LOO) formula, particularly kernel ridge regression (KRR), we developed a computationally tractable procedure that updates all calibration residuals and predictions retrospectively, rather than incrementally appending new ones.  Extensive numerical experiments on both synthetic and real-world data demonstrated that our approach yields faster coverage recovery and tighter prediction intervals than conventional ACI-based methods in online learning with distribution shifts. In addition, the proposed method maintains competitive predictive interval widths during stable periods, incurring little to no additional cost in statistical efficiency. This indicates that the improved adaptivity is achieved without sacrificing performance for stable data streams.

Despite its advantages, the proposed method has several limitations that merit discussion. First, although the proposed method reduces the complexity of retrospective adjustment from $O(n_t^4)$ to $O(n_t^2)$ per step, it remains more computationally demanding than standard online learning algorithms with $O(n_t)$ or constant-time updates. In particular, for high-frequency data streams or large window sizes, the quadratic complexity may still pose a computational bottleneck. As our goal is to achieve improved statistical calibration through retrospective adjustment in a computationally tractable manner rather than to compete with lightweight streaming models in terms of raw throughput, the proposed approach is particularly suitable for applications where statistical reliability and adaptive calibration are prioritized over minimal computational cost. Second, the proposed approach incurs an $O(n_t^2)$ memory cost due to the need to store and update the inverse kernel matrix (or its factorized form). This may become a bottleneck in applications requiring large window sizes, particularly for long-memory processes. 

An important direction for future work is the development of data-driven strategies for selecting the window size $w$. As demonstrated in \cref{appendix:window_size}, the choice of $w$ has a significant impact on the performance: larger windows tend to produce more conservative predictions with wider intervals, while smaller windows can deviates from the nominal level. In the current framework, $w$ is fixed a priori, which limits the applicability of our method. Designing adaptive mechanisms that adjust the window size based on observed data could further improve both robustness and efficiency. Another promising direction is to incorporate the notion of LOO stability \citep{lee2025leave} into our framework, which generalizes the exact LOO formulas to approximate counterparts. By leveraging this concept, our method could be extended beyond linear smoothers to a broader class of regression methods that are stable under small data perturbations. This would further improve the flexibility and applicability of our approach for more complex learning tasks. Lastly, it is also worthwhile to extend the retrospective adjustment principle to efficient approximations of kernel ridge regression, such as random Fourier features \citep{rahimi2007random}, Nyström approximations \citep{williams2001using}, or kernel recursive least squares \citep{van2012kernel}, to further enhance scalability in high-dimensional and large-scale online settings.

\section*{Acknowledgement}
This work was supported by the National Research Foundation of Korea (NRF) funded by the Korea government (MSIT) (RS-2024-00411853 and RS-2026-25486448) and INHA UNIVERSITY Research Grant.

\bibliographystyle{elsarticle-num-names}
\bibliography{_references}

@book{vovk2005algorithmic,
  title={Algorithmic Learning in a Random World},
  author={Vovk, Vladimir and Gammerman, Alexander and Shafer, Glenn},
  year={2005},
  publisher={Springer}
}

@article{yang2024doubly,
  title={Doubly robust calibration of prediction sets under covariate shift},
  author={Yang, Yachong and Kuchibhotla, Arun Kumar and Tchetgen Tchetgen, Eric},
  journal={Journal of the Royal Statistical Society Series B: Statistical Methodology},
  volume={86},
  number={4},
  pages={943--965},
  year={2024},
  publisher={Oxford University Press US}
}

@inproceedings{chernozhukov2018exact,
  title={Exact and robust conformal inference methods for predictive machine learning with dependent data},
  author={Chernozhukov, Victor and W{\"u}thrich, Kaspar and Yinchu, Zhu},
  booktitle={Conference On learning theory},
  pages={732--749},
  year={2018},
  organization={PMLR}
}

@article{shafer2008tutorial,
  title={A tutorial on conformal prediction.},
  author={Shafer, Glenn and Vovk, Vladimir},
  journal={Journal of Machine Learning Research},
  volume={9},
  number={3},
  pages={371--421},
  year={2008}
}

@article{lei2018distribution,
  title={Distribution-free predictive inference for regression},
  author={Lei, Jing and G’Sell, Max and Rinaldo, Alessandro and Tibshirani, Ryan J and Wasserman, Larry},
  journal={Journal of the American Statistical Association},
  volume={113},
  number={523},
  pages={1094--1111},
  year={2018},
  publisher={Taylor \& Francis}
}

@article{barber2021predictive,
  title={Predictive inference with the jackknife+},
  author={Barber, Rina Foygel and Cand{\`e}s, Emmanuel J and Ramdas, Aaditya and Tibshirani, Ryan J},
  journal={The Annals of Statistics},
  volume={49},
  number={1},
  pages={486--507},
  year={2021},
  publisher={JSTOR}}

@article{barber2023conformal,
  title={Conformal prediction beyond exchangeability},
  author={Barber, Rina Foygel and Cand{\`e}s, Emmanuel J and Ramdas, Aaditya and Tibshirani, Ryan J},
  journal={The Annals of Statistics},
  volume={51},
  number={2},
  pages={816--845},
  year={2023}
}

@article{redmond2002data,
  title={A data-driven software tool for enabling cooperative information sharing among police departments},
  author={Redmond, Michael and Baveja, Alok},
  journal={European Journal of Operational Research},
  volume={141},
  number={3},
  pages={660--678},
  year={2002},
  publisher={Elsevier}
}

@article{gibbs2024dtaci,
  author  = {Isaac Gibbs and Emmanuel J. Cand{{\`e}}s},
  title   = {Conformal Inference for Online Prediction with Arbitrary Distribution Shifts},
  journal = {Journal of Machine Learning Research},
  year    = {2024},
  volume  = {25},
  number  = {162},
  pages   = {1--36},
}

@inproceedings{gibbs2021aci,
 author = {Gibbs, Isaac and Emmanuel J. Cand{{\`e}}s},
 booktitle = {Advances in Neural Information Processing Systems},
 pages = {1660--1672},
 title = {Adaptive Conformal Inference Under Distribution Shift},
 volume = {34},
 year = {2021}
}

@article{ikonomovska2011learning,
  title={Learning model trees from evolving data streams},
  author={Ikonomovska, Elena and Gama, Joao and D{\v{z}}eroski, Sa{\v{s}}o},
  journal={Data Mining and Knowledge Discovery},
  volume={23},
  number={1},
  pages={128--168},
  year={2011},
  publisher={Springer}
}

@article{van2012kernel,
  title={Kernel recursive least-squares tracker for time-varying regression},
  author={Van Vaerenbergh, Steven and L{\'a}zaro-Gredilla, Miguel and Santamar{\'\i}a, Ignacio},
  journal={IEEE Transactions on Neural Networks and Learning Systems},
  volume={23},
  number={8},
  pages={1313--1326},
  year={2012},
  publisher={IEEE}
}

@inproceedings{zaffran2022adaptive,
  title={Adaptive conformal predictions for time series},
  author={Zaffran, Margaux and F{\'e}ron, Olivier and Goude, Yannig and Josse, Julie and Dieuleveut, Aymeric},
  booktitle={International Conference on Machine Learning},
  pages={25834--25866},
  year={2022},
  organization={PMLR}
}

@inproceedings{bhatnagar2023improved,
  title={Improved online conformal prediction via strongly adaptive online learning},
  author={Bhatnagar, Aadyot and Wang, Huan and Xiong, Caiming and Bai, Yu},
  booktitle={International Conference on Machine Learning},
  pages={2337--2363},
  year={2023},
  organization={PMLR}
}

@inproceedings{almeida2013adaptive,
  title={Adaptive model rules from data streams},
  author={Almeida, Ezilda and Ferreira, Carlos and Gama, Joao},
  booktitle={Joint European Conference on Machine Learning and Knowledge Discovery in Databases},
  pages={480--492},
  year={2013},
  organization={Springer}
}

@inproceedings{jacot2018neural,
  title={Neural tangent kernel: Convergence and generalization in neural networks},
  author={Jacot, Arthur and Gabriel, Franck and Hongler, Cl{\'e}ment},
  booktitle ={Advances in Neural Information Processing Systems},
  volume={31},
  pages = {8580--8589},
  year={2018}
}

@inproceedings{jun2017improved,
  title={Improved strongly adaptive online learning using coin betting},
  author={Jun, Kwang-Sung and Orabona, Francesco and Wright, Stephen and Willett, Rebecca},
  booktitle={Artificial Intelligence and Statistics},
  pages={943--951},
  year={2017},
  organization={PMLR}
}

@article{wintenberger2017optimal,
  title={Optimal learning with {Bernstein} online aggregation},
  author={Wintenberger, Olivier},
  journal={Machine Learning},
  volume={106},
  number={1},
  pages={119--141},
  year={2017},
  publisher={Springer}
}

@inproceedings{rahimi2007random,
  title={Random features for large-scale kernel machines},
  author={Rahimi, Ali and Recht, Benjamin},
  booktitle={Advances in Neural Information Processing Systems},
  volume={20},
  pages={1177--1184},
  year={2007}
}

@inproceedings{williams2001using,
  title={Using the {Nystr{\"o}m} method to speed up kernel machines},
  author={Williams, Christopher KI and Seeger, Matthias},
  booktitle={Advances in Neural Information Processing Systems},
  volume={13},
  pages={682--688},
  year={2001}
}

@inproceedings{lee2025leave,
  title={Leave-One-Out Stable Conformal Prediction},
  author={Lee, Kiljae and Zhang, Yuan},
  booktitle={The Thirteenth International Conference on Learning Representations},
  year={2025}
}

@article{wendland1995piecewise,
  title={Piecewise polynomial, positive definite and compactly supported radial functions of minimal degree},
  author={Wendland, Holger},
  journal={Advances in Computational Mathematics},
  volume={4},
  number={1},
  pages={389--396},
  year={1995},
  publisher={Springer}
}

@book{fan2020statistical,
  title={Statistical Foundations of Data Science},
  author={Fan, Jianqing and Li, Runze and Zhang, Cun-Hui and Zou, Hui},
  year={2020},
  publisher={Chapman and Hall/CRC}
}

@article{harries1999splice,
  title={Splice-2 comparative evaluation: Electricity pricing},
  author={Harries, Michael and Wales, New South and others},
  year={1999},
  publisher={University of New South Wales, School of Computer Science and Engineering}
}

@article{lu2002inverses,
  title={Inverses of 2$\times$ 2 block matrices},
  author={Lu, Tzon-Tzer and Shiou, Sheng-Hua},
  journal={Computers \& Mathematics with Applications},
  volume={43},
  number={1-2},
  pages={119--129},
  year={2002},
  publisher={Elsevier}
}

@misc{wijffels2025rmoa,
  author       = {Wijffels, Jan},
  year         = {2025},
  title        = {{RMOA: Connect R with MOA for Massive Online Analysis}},
  howpublished = {\textit{R package}},
  note         = {Version 1.1.0},
}

@inproceedings{lee2019wide,
  title={Wide neural networks of any depth evolve as linear models under gradient descent},
  author={Lee, Jaehoon and Xiao, Lechao and Schoenholz, Samuel and Bahri, Yasaman and Novak, Roman and Sohl{-}Dickstein, Jascha and Pennington, Jeffrey},
  booktitle ={Advances in Neural Information Processing Systems},
  volume={32},
  pages = {8570--8581},
  year={2019}
}

@misc{nugent2018sp500,
  author       = {Cam Nugent},
  title        = {{S\&P 500 stock data}},
  year         = {2018},
  howpublished = {Kaggle},
  url          = {https://www.kaggle.com/datasets/camnugent/sandp500},
  note         = {Accessed: 2025-11-06}
}

@article{Liu2018Acc,
  author  = {Liu, Anjin and Lu, Jie and Liu, Feng and Zhang, Guangquan},
  title   = {Accumulating regional density dissimilarity for concept drift detection in data streams},
  journal = {Pattern Recognition},
  volume  = {76},
  pages   = {256--272},
  year    = {2018},
  issn    = {0031-3203},
  doi     = {10.1016/j.patcog.2017.11.009}
}

@article{raza2015ewma,
author = {Raza, Haider and Prasad, Girijesh and Li, Yuhua},
year = {2015},
month = {03},
pages = {659–669},
title = {EWMA model based shift-detection methods for detecting covariate shifts in non-stationary environments},
volume = {48},
journal = {Pattern Recognition},
doi = {10.1016/j.patcog.2014.07.028}
}

@article{hand1997statistical,
  title={Statistical classification methods in consumer credit scoring: a review},
  author={Hand, David J and Henley, William E},
  journal={Journal of the royal statistical society: series a (statistics in society)},
  volume={160},
  number={3},
  pages={523--541},
  year={1997},
  publisher={Wiley Online Library}
}

@inproceedings{tibshirani2019conformal,
  title={Conformal prediction under covariate shift},
  author={Tibshirani, Ryan J and Foygel Barber, Rina and Cand{{\`e}}s, Emmanuel and Ramdas, Aaditya},
  booktitle ={Advances in Neural Information Processing Systems},
  volume={32},
  pages={2526--2536},
  year={2019}
}

@article{mourtada2021amf,
  title={{AMF}: Aggregated Mondrian forests for online learning},
  author={Mourtada, Jaouad and Ga{\"\i}ffas, St{\'e}phane and Scornet, Erwan},
  journal={Journal of the Royal Statistical Society Series B: Statistical Methodology},
  volume={83},
  number={3},
  pages={505--533},
  year={2021},
  publisher={Oxford University Press}
}

@article{montiel2021river,
  title   = {River: machine learning for streaming data in Python},
  author  = {Montiel, Jacob and Halford, Max and Mastelini, Saulo Martiello and Bolmier, Geoffrey and Sourty, Rapha{\"e}l and Vaysse, Robin and Zouitine, Adil and Gomes, Heitor Murilo and Read, Jesse and Abdessalem, Talel and Bifet, Albert},
  journal = {Journal of Machine Learning Research},
  volume  = {22},
  number  = {110},
  pages   = {1--8},
  year    = {2021}
}

@article{wu2021top,
  title={Top-k self-adaptive contrast sequential pattern mining},
  author={Wu, Youxi and Wang, Yuehua and Li, Yan and Zhu, Xingquan and Wu, Xindong},
  journal={IEEE Transactions on Cybernetics},
  volume={52},
  number={11},
  pages={11819--11833},
  year={2021},
  publisher={IEEE}
}

@article{zhang2026one,
  title={One-pass online learning from data streams with unpredictable feature evolution},
  author={Zhang, Peng and Yin, Hongpeng and Zhou, Han},
  journal={Pattern Recognition},
  volume={171},
  pages={112003},
  year={2026},
  publisher={Elsevier}
}

@article{zhang2025one,
  title={One-Pass Online Learning Under Feature Evolution Data Streams With a Fast Rate},
  author={Zhang, Peng and Yin, Hongpeng and Deng, Xuanhong and Lv, Sheng-Qing},
  journal={IEEE Transactions on Knowledge and Data Engineering},
  year={2025},
  publisher={IEEE}
}

\begin{appendices}
\crefalias{section}{appendix}
\crefalias{subsection}{appendix}

\section{ACI-based Algorithms}
\label[appendix]{appendix:aci_algorithms}

\cref{alg:aci} describes the general ACI procedure. 

\begin{algorithm}[H]
\caption{Adaptive Conformal Inference (ACI)}
\label{alg:aci}
\begin{algorithmic}[1]
\State \textbf{Input:} target miscoverage level $\alpha$, starting value $\alpha_1$, step size $\gamma > 0$.
\For{$t = 1,2,\ldots,T$}
  \State \textbf{Return} prediction interval $\widehat{C}_{t}(\alpha_t)$.
  \State Observe $Y_t$.
  \State Update miscoverage level $\alpha_{t+1} = \alpha_t + \gamma \left\{\alpha - \mathbb{I} (Y_t \notin \widehat C_t(\alpha_t))\right\}.$
\EndFor
\end{algorithmic}
\end{algorithm}

\cref{alg:agaci} describes proposed Online Expert Aggregated ACI (AgACI, \citep{zaffran2022adaptive}), which builds on the Bernstein Online Aggregation method of \citet{wintenberger2017optimal} (denoted as BOA in \cref{alg:agaci}).

\begin{algorithm}[H]
\caption{Aggregated Adaptive Conformal Inference (AgACI)}
\label{alg:agaci}
\begin{algorithmic}[1]
\State \textbf{Input:} target miscoverage level $\alpha$, starting value $\alpha_1$, candidate step sizes $\{\gamma_k\}_{k\in[K]}$.
\State Initialize lower and upper BOA algorithms $\mathcal B^L := \textup{BOA}(\alpha \leftarrow (1-\alpha)/2)$
and $\mathcal B^U := \textup{BOA}(\alpha \leftarrow (1-(1-\alpha)/2))$
\State Initialize $\mathcal A_k = \mathrm{ACI}(\alpha,\gamma_k,\alpha_1)$ for $k \in [K]$.
\For{$t=1,2,\ldots,T$}
  \State Retrieve candidate intervals $\widehat C_t(\alpha_t^k)=[L_t^k,U_t^k]$ from $\mathcal A_k$ for $k \in [K]$.
  \State Compute aggregated bounds $\tilde L_t=\mathcal B^L(\{L_t^k\}_{k\in[K]})$ and $\tilde U_t=\mathcal B^U(\{U_t^k\}_{k\in[K]})$.
  \State \textbf{Return} prediction interval $\bigl[ \tilde L_t, \tilde U_t \bigr]$.
  \State Observe $Y_t$ and update $\mathcal A_k$ for $k \in [K]$, $\mathcal B^L$ and  $\mathcal B^U$.
\EndFor
\end{algorithmic}
\end{algorithm}

\cref{alg:dtaci} describes Dynamically-Tuned Adaptive Conformal Inference (DtACI, \citep{gibbs2024dtaci}). DtACI is built upon an alternative perspective of the ACI algorithm, which can be viewed as a gradient descent step applied to the pinball loss defined as $\ell(\theta;\beta) = \alpha (\beta - \theta) - \min \{ 0, \beta - \theta \}$ for $\theta\in\R$ and $\beta\in[0,1]$.  If we define $\beta_t$ as
\begin{align*}
    \beta_t := \sup \{ \beta\in[0,1] : Y_t \in \widehat{C}_t(\beta) \},
\end{align*}
which is the largest miscoverage level such that $Y_t$ lies within $\widehat{C}_t(\beta)$, the ACI update can be equivalently written as  $ \alpha_{t+1} = \alpha_t - \gamma \nabla_{\theta} \ell (\alpha_t;\beta_t).$

\begin{algorithm}[H]
\caption{Dynamically-Tuned Adaptive Conformal Inference (DtACI)}
\label{alg:dtaci}
\begin{algorithmic}[1]
\State \textbf{Input:} target miscoverage level $\alpha$, starting value $\alpha_1$, candidate step sizes $\{\gamma_k\}_{k\in[K]}$, parameters $\{\sigma_t\}_{t\in[T]},\{\eta_t\}_{t\in[T]}$.
\State Initialize $\mathcal A_k = \mathrm{ACI}(\alpha,\gamma_k,\alpha_1)$ for $k \in [K]$.
\For{$t=1,2,\ldots,T$}
  \State  Compute $p_t^k = w_t^k \big/ \sum_{i=1}^K w_t^i$ for all $k\in[K]$.
  \State  Compute  $\alpha_t =\sum_{k=1}^K \alpha_t^k\, p_t^k$.
  \State \textbf{Return} prediction interval $\widehat{C}_{t}(\alpha_t)$.
   \State Observe $Y_t$ and update $\mathcal A_k$ for $k \in [K]$.
  \State Compute $\beta_t$ and set $\bar w_t^k = w_t^k \exp \, \bigl(-\eta_t\,\ell(\alpha_t^k;\beta_t)\bigr)$ for $k\in[K]$.
  \State  Compute $w_{t+1}^k = (1-\sigma_t)w_t^k + \bar W_t\,\sigma_t/K$ with $\bar W_t := \sum_{i=1}^K w_t^i$.
\EndFor
\end{algorithmic}
\end{algorithm}

\citet{bhatnagar2023improved} proposed  Strongly Adaptive Online Conformal Prediction (SAOCP, \cref{alg:saocp}), which uses Scale-Free Online Gradient Descent (SFOGD, \cref{alg:sfogd}) as base experts and then aggregates the experts via a meta-algorithm proposed by \citet{jun2017improved}. Unlike the original SFOGD and SAOCP algorithms, which were developed for width-based constructors, we adapt them here to support the quantile-based construction of prediction sets employed in the proposed method. 

\begin{algorithm}[H]
\caption{Modified version of the SFOGD}
\label{alg:sfogd}
\begin{algorithmic}[1]
\State \textbf{Input:} target miscoverage level $\alpha$, starting value $\alpha_1$, step size $\gamma > 0$.
\For{$t = 1,2,\ldots,T$}
  \State \textbf{Return} prediction interval $\widehat{C}_{t}(\alpha_t)$.
  \State Observe $Y_t$ and compute $\beta_t$.
  \State Update $\alpha_{t+1} = \alpha_t - \gamma \nabla_{\theta} \ell ( \alpha_t;\beta_t)/\sqrt{\sum_{s=1}^{t} \| \nabla_{\theta}\ell(\alpha_s;\beta_s)\|_2^2}$.
\EndFor
\end{algorithmic}
\end{algorithm}

In \cref{alg:saocp},  $\lfloor x \rfloor$ denotes the largest integer less than or equal to $x$ and $[x]_+:=\max\{0,x\}$ does the positive part of $x$ for a real number $x$. Moreover, $\Delta^t:=\{(p_1,\dots, p_t)\in[0,1]^d: \sum_{i=1}^t p_i=1\}$ denotes the $t$-dimensional probability simplex.
 
\begin{algorithm}[H]
\caption{Modified version of the SAOCP}
\label{alg:saocp}
\begin{algorithmic}[1]
\State \textbf{Input:} target miscoverage level $\alpha$, starting value $\alpha_1$, step size $\gamma > 0$, lifetime multiplier $g\in\bN$.
\For{$t = 1,2,\ldots,T$}
  \State Initialize $\mathcal{A}_t = \text{SFOGD}(\alpha \leftarrow \alpha, \gamma \leftarrow \gamma, \alpha_1 \leftarrow \alpha_{t-1})$ with $w_t^t = 0$.
  \State Compute $\mathrm{Active}(t) = \{ i \in [T] : t - L(i) < i \le t \}$ where $L(i) := g \cdot \max_{n \in \mathbb{Z}} \{ 2^n : i \equiv 0 \mod 2^n \}$
  \State Compute probability $\hat{p}_i \propto \pi_i [w_{t}^i]_+$, where $\pi_i \propto i^{-2}(1 + \lfloor \log_2 i \rfloor)^{-1} \ind(i \in \mathrm{Active}(t))$.
  \State Set $\alpha_t = \sum_{i \in \mathrm{Active}(t)} p_i \alpha_t^i$ for $t \ge 2$, and $\alpha_t = 0$ for $t=1$.
  \State \textbf{Return} prediction set $\widehat{C}_{t}(\alpha_t)$.
  \State Observe $Y_t$ and compute $\beta_t$.
  \For{$i \in \mathrm{Active}(t)$}
    \State Update expert $\mathcal{A}_t$ with $Y_t$ and obtain $\alpha_{t+1}^i$.
    \State Compute    $g_t^i=\{\ell(\alpha_t;\beta_t) - \ell(\alpha_t^i;\beta_t)\}\ind(w_t^i > 0)+\big[\ell(\alpha_t;\beta_t) - \ell(\alpha_t^i;\beta_t)\big]_+\ind(w_t^i \le 0) $.
    \State Update expert weight $w_{t+1}^i = \frac{1}{t - i + 1} \left( \sum_{j=i}^t g_j^i \right) \left( 1 + \sum_{j=i}^t w_j^i g_j^i \right).$
  \EndFor
\EndFor
\end{algorithmic}
\end{algorithm}

\section{Proof of \cref{thm:longterm_coverage}}
\label{proof:longterm_coverage}

\paragraph{Proof for ACI and DtACI}

By definition, when $\alpha_t<0$ then $\ind(Y_t \notin\widehat{C}^{\mathrm{RA}}_{t}(\alpha_t))=0$ and when $\alpha_t>1$ then $\ind(Y_t \notin\widehat{C}^{\mathrm{RA}}_{t}(\alpha_t))=1$ always. Hence, the same conclusion of Lemma 4.1 of \citet{gibbs2021aci} follows here and thus the desired result follows by applying the argument of Proposition 4.1 of \citet{gibbs2021aci} for ACI and Theorem 6 of \citet{gibbs2024dtaci} for DtACI. 

\begin{remark}
The proposed RetroAdj method employs empirical quantiles and therefore does not satisfy the continuity assumption used in the idealized ACI formulation of \citet{gibbs2021aci}. However, this observation does not invalidate our theoretical result, since \cref{thm:longterm_coverage} in our paper relies only on the long-run coverage argument corresponding to Lemma 4.1 and Proposition 4.1 of \citet{gibbs2021aci}, rather than on the later marginal-coverage deviation analysis, which requires the continuity of the quantile function.
\end{remark}

\paragraph{Proof for SFOGD}

\cite{bhatnagar2023improved} originally proposed the SFOGD algorithm as a width-based interval constructor as we have explained in \cref{appendix:aci_algorithms}. They assumed that the interval radius $\varphi_t$ remain bounded, and proved the coverage guarantee under setting by iteratively updating the radius $\varphi_t$. In contrast, our method dynamically adjusts the quantile parameter $\alpha_t$ instead of the radius, and thus we slightly modify the proof accordingly. Let $\textup{err}_t=\ind(Y_t \notin\widehat{C}^{\mathrm{RA}}_{t}(\alpha_t))$ to simplify the notation.

\begin{lemma} (Boundedness of quantile parameter $\alpha_t$ in SFOGD)
\label{lemma::sfogd_bdd}
For any $t \in \mathbb{N},$ we have $\alpha_t \in [-\gamma, 1 + \gamma]$   with probability one.
\end{lemma}
\begin{proof}
    Note that $\sup_t |\alpha_{t+1} - \alpha_t| = \sup_t \gamma \left| (\textup{err}_t -  \alpha)/\sqrt{\sum_{s=1}^t (\textup{err}_s - \alpha)^2} \right| \leq \gamma.$
Thus, the desired result follows by the same argument of the proof of Lemma 4.1 of \cite{gibbs2021aci}.
\end{proof}

With the above lemma in hand, we can obtain the following non-asymptotic error bound on the long-run coverage, which concludes the desired result for the SFOGD.

\begin{theorem}[Modified version of Theorem 4.2 of \cite{bhatnagar2023improved}]
\label{thm::sfogd_cov}
\cref{alg:sfogd} with any learning rate $\gamma = \Theta(1)$ and any initialization $\alpha_1 \in (0, 1)$ achieves $ \left| \frac{1}{T}\sum_{t=1}^T \textup{err}_t -\alpha \right| \leq O(\alpha^{-2}T^{-1/4}\log T)$ with probability one.
\end{theorem}

\begin{proof}
Since $ \nabla_{\theta}\ell(\alpha_t;\beta_t)  = \textup{err}_t - \alpha,$ 
the SFOGD update rule can be expressed as
\begin{align*}
    \alpha_{t+1} = \alpha_t +\gamma \frac{\textup{err}_t - \alpha}{\sqrt{\sum_{s=1}^{t}(\textup{err}_s-\alpha)^2}} = \alpha_1 + \gamma \sum_{s=1}^t \frac{\textup{err}_s - \alpha}{\sqrt{\sum_{i=1}^{s}(\textup{err}_i-\alpha)^2}}
\end{align*}
Note that we have $\alpha_{t+1} \in [-\gamma, 1+\gamma]$ for all $t \geq 0$ by \cref{lemma::sfogd_bdd}, which implies that
\begin{align*}
      \left| \sum_{t=t_{0}+1}^{t_f} \frac{\textup{err}_t - \alpha}{\sqrt{\sum_{s=1}^{t}(\textup{err}_s-\alpha)^2}} \right| = \frac{1}{\gamma}|\alpha_{t_{f}+1} - \alpha_{t_{0}+1}| \leq \frac{1 + 2\gamma}{\gamma}
\end{align*}
Therefore, 
by Lemma B.2 of Appendix B in \cite{bhatnagar2023improved} with  $a_t = \textup{err}_t - \alpha$ and $M = \frac{1 +2\gamma}{\gamma}$, we have
\begin{align*}
    \left| \frac{1}{T}\sum_{t=1}^T \textup{err}_t -\alpha \right| \leq 2((1 +3\gamma)/\gamma + \alpha^{-2}\log T)T^{-1/4} =O(\alpha^{-2}T^{-1/4}\log T)
\end{align*}
for any $\gamma = \Theta(1).$
\end{proof}

\paragraph{Proof for SAOCP} For the long-run coverage result for the proposed method applied with the SAOCP, we need a suitable assumption on the quantity $S_{\beta}(T)$, the measure of smoothness of the expert weights and the cumulative gradient norms for each individual expert. See Theorem B.3 of \cite{bhatnagar2023improved} for detailed definition of $S_{\beta}(T)$. Due to the theorem given below, if there exists $\beta \in (1/2, 1)$ such that $S_{\beta}(T) = O(T^{\xi})$ for some  $\xi \in(0, 1-\beta)$ up to a polylogarithmic factor, then we get the desired result.

\begin{theorem}[Modified version of Theorem 4.3 of \cite{bhatnagar2023improved}] Consider a modified version of \cref{alg:saocp} where line 8 is replaced by sampling an expert $i \sim p$. Then, for any learning rate $\gamma = \Theta(1)$ and any initialization $\alpha_1 \in (0, 1)$, we have
    \begin{align*}
    \left| \frac{1}{T} \sum_{t=1}^T \E[\textup{err}_t] - \alpha \right| = O\del{\inf_{\beta \in (1/2, 1)} \left\{ T^{1/2-\beta} + T^{\beta-1} S_{\beta}(T) \right\}}
    \end{align*} 
with probability one, where the expectation is taken over the randomness of sampling an expert.
\end{theorem}

\begin{proof}
By \cref{lemma::sfogd_bdd}, we have $ |\alpha_{t+1} - \alpha_t| /\gamma\leq (1 + 2\gamma)/\gamma$
Since $(1 + 2\gamma)/\gamma$ is a fixed constant independent of $t$, the sequence $\{ \alpha_t \}_{t \in [T]}$ is uniformly bounded in its increments. Hence, the proof follows by the same argument of the proof of Theorem B.3 of Appendix B in \cite{bhatnagar2023improved}. 
\end{proof}

\section{Proofs of \cref{lemma:loopred,lemma:krr_downdate,lemma:krr_update}}

\subsection{Proof of \cref{lemma:loopred}}
\label{proof:loopred}

Define $Y_i^\dag:=\hat{f}_{[n]\setminus \{i\}}(X_i)$ and $Y_j^\dag:=Y_j $ for $j\neq i$. Then by the self-stable property, $\hat{f}_{[n]\setminus \{i\}}$ is equal to the linear smoother, say $\hat{f}_{[n]}^\dag$, trained on the ``perturbed'' data set $\{(X_j, Y_j^\dag)\}_{j\in[n]}$. Since the feature vectors of this perturbed data set are equal to those of the original data set $\{(X_j, Y_j)\}_{j\in[n]}$, we have
    \begin{align*}
        \hat{f}_{[n]\setminus \{i\}}(x)
        =\hat{f}_{[n]}^\dag(x)
        &=\xi_n(x,X_{1:n})^\top Y_{1:n}^\dag\\
        &=\xi_n(x,X_{1:n})^\top Y_{1:n}+\xi_{n}^{i}(x)(Y_i^\dag-Y_i)\\
          &=\hat{f}_{[n]}(x)-\xi_{n}^{i}(x)(Y_i- \hat{f}_{[n]\setminus \{i\}}(X_i)).
    \end{align*}
\cref{lemma:looresid} concludes the result.

\subsection{Proof of \cref{lemma:krr_downdate}}
\label{proof:krr_downdate}
Let $\kappa_o := \kappa(X_{t-1-w}, X_{t-1-w}) $ and $  u_o := (\kappa(X_{t-1-w},X_{i}))_{i\in \cI(t)}.$ 
Then, the regularized kernel matrix $H^{(t)} := K^{(t)} +\lambda I$ can be partitioned as
\begin{align*}
    H^{(t)} = K^{(t)} + \lambda I = \begin{pmatrix}
        \kappa_{o} & u_{o}^\top \\
        u_{o} & \check H^{(t)}
    \end{pmatrix}, \quad (H^{(t)})^{-1} = Q^{(t)} = 
    \begin{pmatrix}
        q_{11} & q_{12}^\top \\
        q_{12} & Q_{22}
    \end{pmatrix},
\end{align*}
where $\check H^{(t)} := \check K^{(t)} + \lambda I$. By the block matrix inversion formula  \citep[e.g., Theorem 2.1 of ][]{lu2002inverses}, we have
\begin{align*}
    (H^{(t)})^{-1} =
    \begin{pmatrix}
        \delta_{o}^{-1} & -\delta_{o}^{-1} u_{o}^\top (\check H^{(t)})^{-1} \\[5pt]
        -(\check H^{(t)})^{-1} u_{o} \, \delta_{o}^{-1} &
        (\check H^{(t)})^{-1}
        + (\check H^{(t)})^{-1}
          u_{o} \, \delta_{o}^{-1} u_{o}^\top
          (\check H^{(t)})^{-1}
    \end{pmatrix},
\end{align*}
where we define $\delta_o := \kappa_o - u_o^{\top}(\check{H}^{(t)})^{-1}u_o$.
From the result, it follows that
\begin{align*}
    q_{11} &= \delta_o^{-1}, \\
    q_{12} &= -(\check H^{(t)})^{-1} u_o \, \delta_o^{-1} \\
    Q_{22} &= (\check H^{(t)})^{-1} + (\check H^{(t)})^{-1} u_{o} \, \delta_{o}^{-1} u_{o}^\top (\check H^{(t)})^{-1}.
\end{align*}
Rearranging the above identity yields
\begin{align*}
    Q_{22} = (\check H^{(t)})^{-1} + \frac{1}{q_{11}}q_{12} q_{12}^\top,
\end{align*}
from which the desired result immediately follows.

\subsection{Proof of \cref{lemma:krr_update}}
\label{proof:krr_update}
The regularized kernel matrix $K^{(t+1)} + \lambda I$ can be partitioned as
\begin{align*}
    K^{(t+1)} + \lambda I = \begin{pmatrix}
        \check K^{(t)}+\lambda I & \uk \\
        \uk & 1 + \lambda
    \end{pmatrix}.
\end{align*}
Applying the block matrix inversion formula \citep[e.g., Theorem 2.1 of ][]{lu2002inverses} to the partition above immediately yields the desired result.

\section{Implementation Details}
\label{appendix:implementation}

The hyperparameters for FIMT-DD and AMRules were set to the default values provided in RMOA package \citep{wijffels2025rmoa}, while those for AMF were set to the default values in the \texttt{river} package \citep{montiel2021river}. Since FIMT-DD, AMRules and AMF are all designed for online learning, hyperparameter tuning is not generally critical, as these methods adapt automatically to evolving data streams. The hyperparameters for the ACI-based algorithms were set to the default values provided in the original papers. We fix the step size $\gamma = 0.005$ for ACI, SFOGD and SAOCP, while consider multiple values $\gamma\in\{0.001, 0.002, 0.004, 0.008, 0.016, 0.032, 0.064, 0.128 \}$ for AgACI and DtACI. For DtACI, we set $\sigma_t = 1/2L$, $\eta_t = \sqrt{\frac{\log(8L)+2}{\sum_{s=t-L}^{t-1} \ell (\alpha_s;\beta_s)}}$ where $ L:= T - t_{\textup{init}}$. For SAOCP, we fix lifetime multiplier $g = 8$.
All experiments were conducted in \texttt{R} (version 4.4.3), using \texttt{RMOA} version 1.1.0 and \texttt{rJava} version 1.0-11, except for AMF, which was implemented in \texttt{Python} (version 3.12.7) using \texttt{river} version 0.23.0. The experiments were run on a machine equipped with an Intel Core i9-13900K CPU and 64\,GB RAM. No additional scaling was applied to Elec2 or AIG, and the Communities and Crime data were used as provided with numeric attributes normalized to $[0,1]$.

\section{Additional Experiments}

\subsection{The Effect of Window Size}
\label{appendix:window_size}

\cref{fig:window_sensitivity} illustrates the effect of the window size $w$ on coverage and interval width. We observe that when $w$ is too small, the empirical coverage deviates from the nominal level, indicating insufficient calibration due to limited sample size in the calibration set. We also observe that when $w$ is too large, the method tends to produce more conservative predictions with wider intervals. This behavior is likely due to the inclusion of outdated observations, which biases the calibration set away from the current distribution. In contrast, moderate window sizes yield coverage closer to the nominal level while maintaining shorter intervals, indicating a better balance between calibration and efficiency.  Overall, RetroAdj performs considerably better than FW-KRR even at relatively small window sizes.

\begin{figure}[H]
    \centering
    \includegraphics[width=0.9\textwidth]{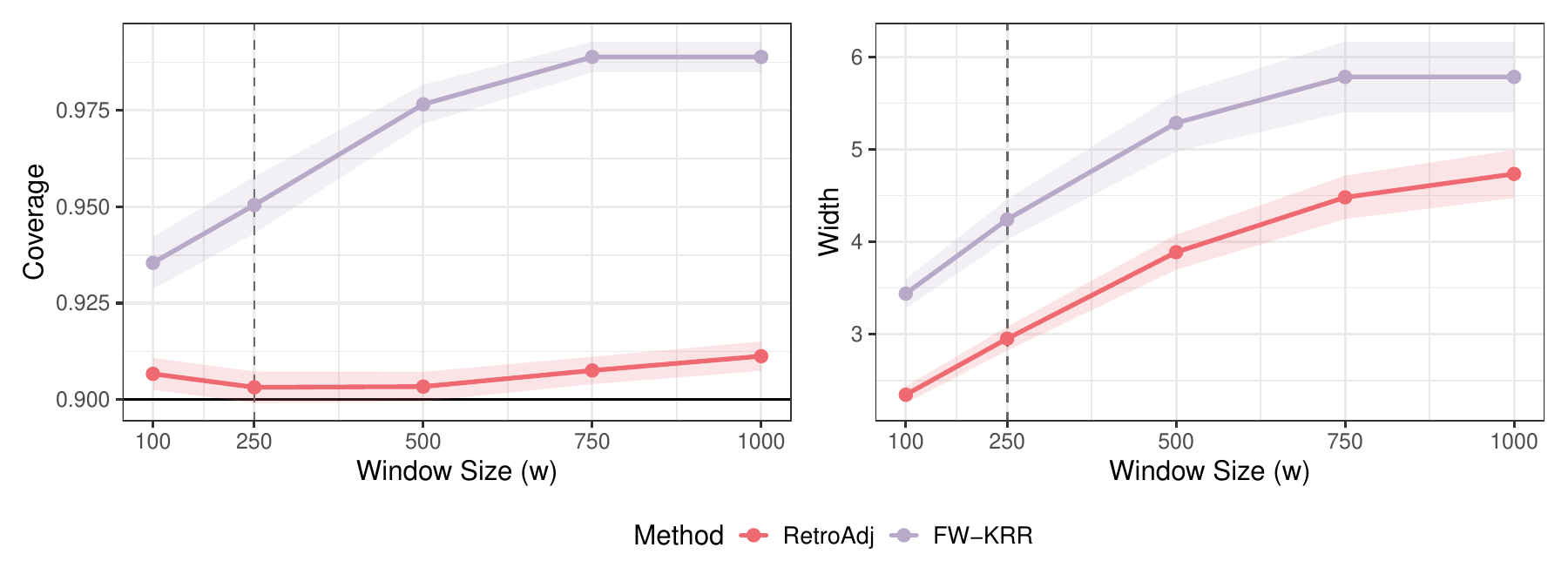}
    \caption{Coverage and prediction interval width in Setting 1 across window sizes $w$. Results are averaged over 50 replications; lines and points denote the mean, ribbons denote $\pm 1$ standard deviation, and the vertical dashed line marks $w=250$.}
    \label{fig:window_sensitivity}
\end{figure}

\subsection{Run-Time Comparison}

\cref{tab:runtime_elec2} reports the average runtime per update across different window sizes on the Elec2 dataset. We observe that the naive retrospective implementation (RetroAdj-naive), which recomputes leave-one-out residuals via repeated retraining, quickly becomes computationally infeasible as the window size increases, resulting in timeouts for moderate values of $w$. This highlights the prohibitive cost of naive retrospective adjustment. In contrast, the proposed RetroAdj method remains computationally tractable across all window sizes. 
While its runtime increases with $w$, the method scales in a controlled manner and remains practical for moderate window sizes.

Forward-update methods exhibit relatively stable runtime across window sizes. However, at practically relevant window sizes (e.g., $w=250$ used in our experiments), the runtime of RetroAdj is comparable to FW-KRR, and substantially lower than the other forward methods. Overall, these results demonstrate that RetroAdj provides a favorable trade-off: it enables retrospective adjustment at a computational cost that is dramatically lower than naive implementations, while remaining competitive with forward-update methods in practical settings.

\begin{table}[H]
\centering
\caption{Average runtime (ms/update) across window sizes on the Elec2 dataset. Runtime was measured as wall-clock elapsed time, excluding one-time initialization costs, and each estimate was averaged over 5 repetitions after one untimed warm-up update and a timed block of 50 consecutive updates; runs exceeding 180 seconds were recorded as timeouts.}
\label{tab:runtime_elec2}
\footnotesize
\setlength{\tabcolsep}{4pt}
\renewcommand{\arraystretch}{1.05}
\begin{tabular}{lccccc}
\hline
Method & $w = 100$ & $w = 250$ & $w = 500$ & $w = 750$ & $w = 1000$ \\
\hline
RetroAdj
& 0.54 (0.14) & 1.43 (0.39) & 4.78 (0.44) & 12.32 (1.09) & 26.78 (3.55) \\
RetroAdj-naive
& 128.9 (12.7) & 2903.2 (331.6) & timeout & timeout & timeout \\
FW-KRR
& 0.80 (0.06) & 1.04 (0.44) & 0.75 (0.16) & 0.82 (0.10) & 0.90 (0.19) \\
FW-AMF
& 1.91 (0.17) & 2.12 (0.27) & 2.37 (0.30) & 2.36 (0.37) & 2.35 (0.30) \\
FW-FIMT-DD
& 31.47 (2.69) & 30.56 (3.21) & 31.16 (3.55) & 32.57 (3.26) & 34.06 (4.79) \\
FW-AMRules
& 32.20 (3.51) & 30.91 (3.24) & 32.22 (3.45) & 34.45 (6.17) & 33.06 (3.44) \\
\hline
\end{tabular}
\end{table}

\subsection{Robustness to the Step Size Choice}

We conduct a simulation to illustrate the robustness of the proposed RetroAdj to the choice of the step size in the ACI algorithm. As shown in \cref{fig::robust}, RetroAdj shows superior performance with a single ACI instance, even under the least favorable setting of $\gamma$, outperforming the forward conformal inference method equipped with the tuning strategy of DtACI. RetroAdj maintains coverage consistently around 0.9, whereas FW-KRR tends to over coverage due to excessively wide intervals.

\begin{figure}[H]
    \centering
    \includegraphics[width=0.9\textwidth]{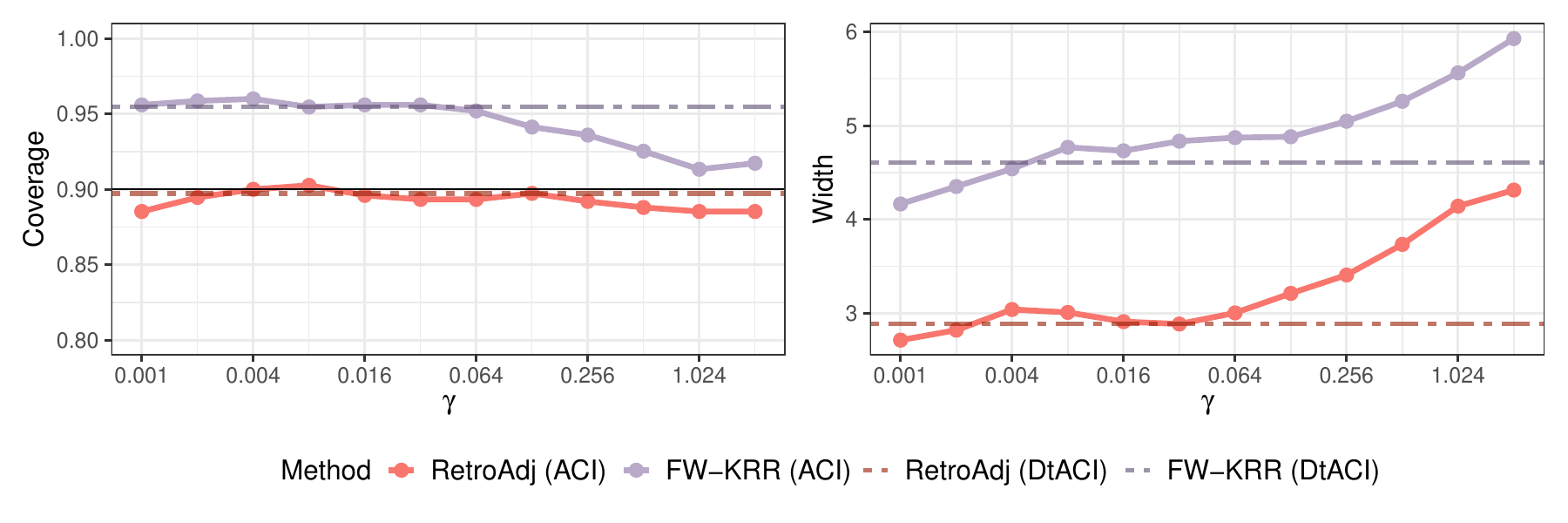}
    \caption{Coverage and prediction interval width of the proposed RetroAdj and forward online conformal inference methods (FW) for Setting 1.}
    \label{fig::robust}
\end{figure}

\subsection{Application to Neural Tangent Kernel}

In this section, we examine whether the proposed method performs consistently well when applied with different kernel function. Specifically, we consider  a Neural Tangent Kernel (NTK) of a two-layer ReLU network  \citep{lee2019wide} which is given by
\begin{align*}
    \kappa(x, y) = \frac{x^{\top}y}{\|x\| \| y\|} (\sin \theta + (\pi - \theta) \cos \theta) + \frac{\pi - \theta}{\pi},
\end{align*}
for $x, y \in \R^d$, where the angle $\theta$ is defined as $\theta := \arccos \left( \frac{x^{\top}y}{\|x\| \| y\|} \right).$ In \cref{fig:ntk}, we observe that the results for the RBF kernel and NTK are almost similar.

\begin{figure}[H]
    \centering
    \includegraphics[width=0.9\textwidth]{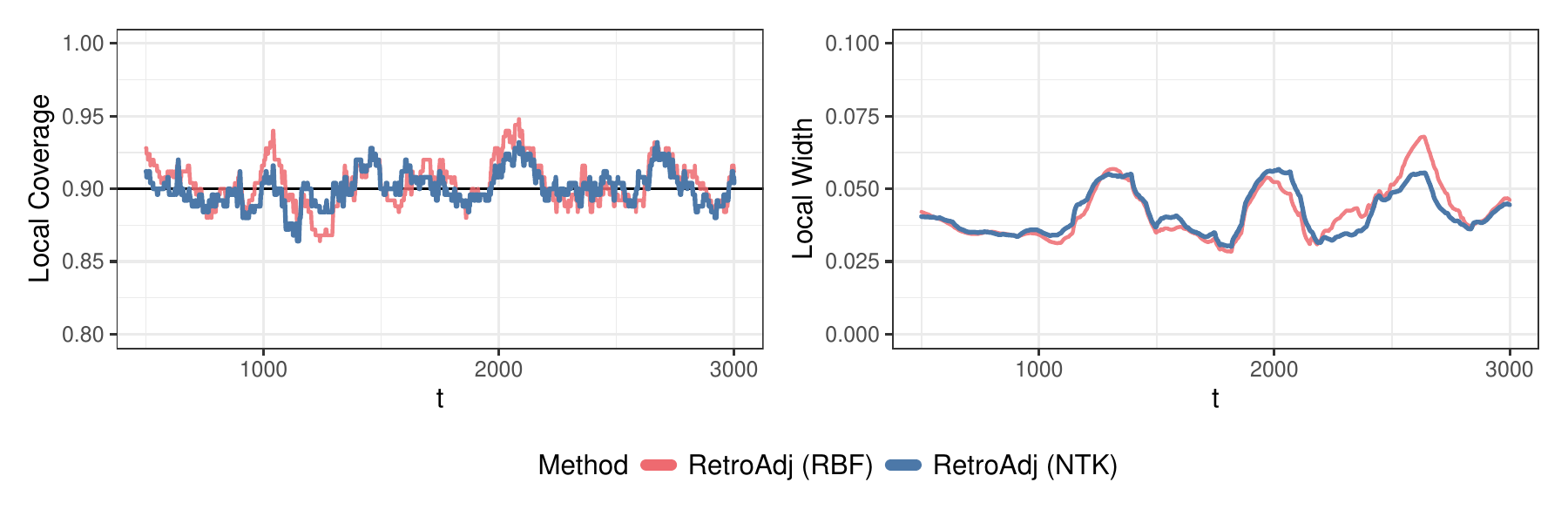}
    \caption{Local coverage and prediction interval width of RetroAdj with the RBF kernel and NTK for the prediction of the Elec2 dataset (Same setting as \cref{sec:realdata}). For both methods, the DtACI algorithm is employed to adjust the miscoverage level.}
    \label{fig:ntk}
\end{figure}

\end{appendices}
\end{document}